\documentclass{article}

\usepackage{microtype}
\usepackage{graphicx}
\usepackage{subfigure}
\usepackage{booktabs} 
\usepackage{xspace}
\usepackage{amsmath}
\usepackage{amsthm}
\usepackage{amssymb}
\usepackage{mathrsfs}
\usepackage{enumitem}
\usepackage{makecell}
\usepackage{hyperref}

\usepackage[accepted]{icml2021}

\makeatletter
\DeclareRobustCommand\onedot{\futurelet\@let@token\@onedot}
\def\@onedot{\ifx\@let@token.\else.\null\fi\xspace}
\def\st{\textrm{s.t}\onedot}
\def\eg{\emph{e.g}\onedot} 
\def\ie{\emph{i.e}\onedot} \def\Ie{\emph{I.e}\onedot}
 
\def\etc{\emph{etc}\onedot}

\makeatother

\newtheorem{theorem}{Theorem}

\newtheorem{lemma}{Lemma}

\newcommand{\fone}{\ensuremath{\mathbf{u}_1\xspace}}
\newcommand{\fthei}{\ensuremath{\mathbf{u}_i\xspace}}

\newcommand{\fd}{\ensuremath{\mathbf{u}_d\xspace}}
\newcommand{\fonetilde}{\ensuremath{\tilde{\mathbf{u}}_1\xspace}}
\newcommand{\ftheitilde}{\ensuremath{\tilde{\mathbf{u}}_i\xspace}}

\newcommand{\fdtilde}{\ensuremath{\tilde{\mathbf{u}}_d\xspace}}
\newcommand{\eone}{\ensuremath{\mathbf{e}_1\xspace}}

\newcommand{\ethei}{\ensuremath{\mathbf{e}_i\xspace}}
\newcommand{\ed}{\ensuremath{\mathbf{e}_d\xspace}}
\newcommand{\qtilde}{\ensuremath{\tilde{q}\xspace}}
\newcommand{\Qtilde}{\ensuremath{\tilde{Q}\xspace}}
\newcommand{\atilde}{\ensuremath{\tilde{a}\xspace}}
\newcommand{\Atilde}{\ensuremath{\tilde{A}\xspace}}
\newcommand{\xbf}[2]{\mathbf{#1}_{#2}}
\newcommand{\inner}[2]{\langle{#1},{#2}\rangle_{\mathcal{H}}}
\definecolor{mycolor}{RGB}{219,90,107}

\icmltitlerunning{Towards Better Laplacian Representation in Reinforcement Learning with Generalized Graph Drawing}

\begin{document}

\twocolumn[
\icmltitle{Towards Better Laplacian Representation in Reinforcement Learning with Generalized Graph Drawing}



\icmlsetsymbol{equal}{*}

\begin{icmlauthorlist}
\icmlauthor{Kaixin Wang}{equal,nus}
\icmlauthor{Kuangqi Zhou}{equal,nus}
\icmlauthor{Qixin Zhang}{cityu}
\icmlauthor{Jie Shao}{bytedance}
\icmlauthor{Bryan Hooi}{nus}
\icmlauthor{Jiashi Feng}{nus}
\end{icmlauthorlist}

\icmlaffiliation{nus}{National University of Singapore}
\icmlaffiliation{cityu}{City University of Hong Kong}
\icmlaffiliation{bytedance}{ByteDance AI lab}

\icmlcorrespondingauthor{Kaixin Wang}{kaixin.wang@u.nus.edu}
\icmlcorrespondingauthor{Kuangqi Zhou}{kzhou@u.nus.edu}


\vskip 0.3in
]



\printAffiliationsAndNotice{\icmlEqualContribution} 

\begin{abstract}
The Laplacian representation recently gains increasing attention for reinforcement learning as it provides succinct and informative representation for states, by taking the eigenvectors of the Laplacian matrix of the state-transition graph as state embeddings.
Such representation captures the geometry of the underlying state space and is beneficial to RL tasks such as option discovery and reward shaping. 
To approximate the Laplacian representation in large (or even continuous) state spaces, recent works propose to minimize a spectral graph drawing objective, which however has \emph{infinitely many} global minimizers other than the eigenvectors.
As a result, their learned Laplacian representation may differ from the ground truth. 
To solve this problem, we reformulate the graph drawing objective into a generalized form and derive a new learning objective, which is proved to have eigenvectors as its \emph{unique} global minimizer. 
It enables learning high-quality Laplacian representations that faithfully approximate the ground truth.
We validate this via comprehensive experiments on a set of gridworld and continuous control environments.
Moreover, we show that our learned Laplacian representations lead to more exploratory options and better reward shaping.
\end{abstract}

\section{Introduction}
\label{sec:intro}
Reinforcement learning (RL) aims to train an agent that can take proper sequential actions based on the perceived states from the environments~\cite{sutton2018reinforcement}.
Thus the quality of state representations is important to the agent performance by benefiting its generalization ability~\cite{zhang2018decoupling, stooke2020decoupling, agarwal2021contrastive}, exploration ability~\cite{pathak2017curiosity, machado2017laplacian, machado2020count}, and learning efficiency~\cite{dubey2018investigating, wu2018laplacian}, \etc.
Recently, the Laplacian representation receives increasing attention~\cite{mahadevan2005proto, machado2017laplacian, wu2018laplacian, jinnai2019exploration}. 
It views the states and transitions in an RL environment as nodes and edges in a graph, and forms a $d$-dimension state representation with $d$ smallest eigenvectors of the graph Laplacian.
Such representations can capture the geometry of the underlying state space, as illustrated in Fig.~\ref{fig:fig1}(b), which greatly benefits RL  in option discovery~\cite{machado2017laplacian, jinnai2019exploration} and reward shaping~\cite{wu2018laplacian}.

\begin{figure}[t]
    \centering
    \includegraphics[width=0.95\linewidth]{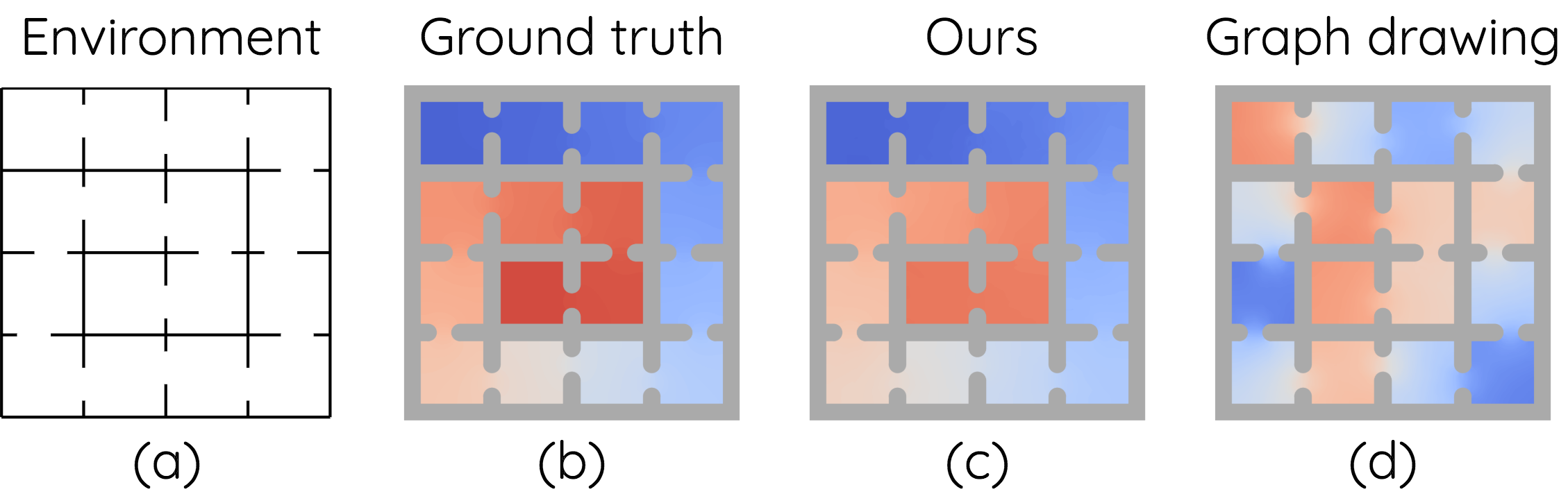}
    \vskip -0.05in
    \caption{Visualization of environment and Laplacian state representations (2nd dimension).
    (a)~Top view of a continuous control navigation environment. (b)~Ground-truth Laplacian representation. It encodes geometry of the environment: nearby states have similar values while distant states have dissimilar values. (c)~Our learned representation. It can be seen that our result is very similar to ground truth. 
    (d)~Representation learned by spectral graph drawing~\cite{wu2018laplacian}. It significantly diverges from ground truth and fails to capture geometric information about the state space. Best viewed in color.
    }
     \label{fig:fig1}
\end{figure}

However, to compute the exact Laplacian representation is very challenging,  as directly computing eigendecomposition of the graph Laplacian requires access to the environment transition dynamics and involves expensive matrix operation.
Hence it is largely limited to small finite state spaces.
To deal with large (or even continuous) state spaces, previous works resort to approximation methods~\cite{machado2017laplacian, machado2018eigenoption, wu2018laplacian}. 
A most efficient one is \cite{wu2018laplacian}, which minimizes a spectral graph drawing objective~\cite{koren2005drawing}.  
However, this objective has infinitely many other global minimizers besides the ground truth Laplacian representation (\ie, $d$ smallest eigenvectors), as it is invariant to an arbitrary orthogonal transformation over the optimization variables. 
The resulted representations can then correspond to other minimizers and diverge from the ground truth, thus unable to encode the geometry of the state space as expected (see Fig.~\ref{fig:fig1}(d)).

To break such invariance and approximate Laplacian representation closer to the ground-truth, 
we reformulate the graph drawing objective into a generalized form by introducing coefficients for each term in it.
By assigning decreasing values for these coefficients, we can derive a training objective that breaks the undesired invariance.
We provide theoretical guarantees that the proposed objective has the $d$ smallest eigenvectors as its \textit{unique} global minimizer under mild assumptions.
As shown in Fig.~\ref{fig:fig1}(c), minimizing the new objective is able to ensure faithful approximation to the ground truth Laplacian representation.

To verify the effectiveness of our method for learning high-quality Laplacian representations, we conduct experiments in gridworld and continuous control environments.  
We show that the learned representations by our method more accurately approximate the ground truth, compared with the ones from the graph drawing objective.
Furthermore, we apply the learned representations to two downstream RL tasks. 
It is demonstrated that, for option discovery task~\cite{machado2017laplacian} our method leads to discovered options that are more exploratory compared to using the representation learned by graph drawing; in reward shaping task~\cite{wu2018laplacian}, our learned representation is better at accelerating agents' learning than previous work~\cite{wu2018laplacian}.

The rest of the paper is organized as follows. In Sec.~\ref{sec:background}, we introduce background about RL and Laplacian representation in RL.
In Sec.~\ref{sec:method}, we propose our new objective for learning Laplacian representation. 
Then, we conduct experiments to demonstrate that our proposed objective is able to learn high-quality representations in Sec.~\ref{sec:exp}. 
In Sec.~\ref{sec:related works} we review related works and Sec.~\ref{sec:conclusion} concludes the paper.

\section{Background}
\label{sec:background}

\subsection{Reinforcement Learning}
In the RL framework~\cite{sutton2018reinforcement}, an agent interacts with an environment by observing states and taking actions, with an aim of maximizing cumulative reward. 
We consider Markov Decision Process (MDP) formalism in this paper. 
An MDP can be described by a 5-tuple $\langle \mathcal{S}, \mathcal{A}, r, p, \gamma\rangle$.
Specifically, at time $t$ the agent observes state $s_t\in \mathcal{S}$ and takes an action $a_t\in\mathcal{A}$. The environment then yields a reward signal $R_t$ sampled from the reward function $r(s_t, a_t)$. The state observation in the next timestep $s_{t+1}\in\mathcal{S}$ is sampled according to an environment-specific transition distribution function $p(s_{t+1}|s_t, a_t)$. A policy is defined as a mapping $\pi: \mathcal{S}\to\mathcal{A}$ that returns an action $a$ given a state $s$. The goal of the agent is to learn an optimal policy $\pi^*$ that maximizes the expected cumulative reward:
\begin{equation}
    \pi^* = \arg\max_{\pi\in\Pi}\mathbb{E}_{p,\pi}\sum_{t=0}^{\infty}\gamma^t R_t,
\end{equation}
where $\Pi$ denotes the policy space and $\gamma\in[0, 1)$ is a discount factor.

\subsection{Laplacian Representation in RL}
\label{subsec: background: laprep}

By considering states and transitions in an MDP as nodes and edges in a graph, the Laplacian state representation is formed by the $d$ smallest eigenvectors of the graph Laplacian. Specifically, each eigenvector (of length $|\mathcal{S}|$) corresponds to a dimension of the Laplacian representation for all states. Formally, we denote the graph as $\mathcal{G}=(\mathcal{S}, \mathcal{E})$ where $\mathcal{E}$ is the edge set consisting of transitions between states. The graph Laplacian of $\mathcal{G}$ is defined as $L=D-A$, where $A \in \mathbb{R}^{|\mathcal{S}|\times|\mathcal{S}|}$ is the adjacency matrix of  $\mathcal{G}$, and $D=\mathrm{diag}(A\mathbf{1})$ is the degree matrix~\cite{chung1997spectral}. We denote the $i$-th smallest eigenvalue of $L$ as $\lambda_i$, and the corresponding unit eigenvector as $\ethei \in \mathbb{R}^{|\mathcal{S}|}$. The $d$-dimensional Laplacian representation of a state $s$ is $\varphi(s)=(\eone[s], \cdots, \ed[s])$, where $\ethei[s]$ denotes the entry in vector $\ethei$ that corresponds to state $s$. In particular, $\eone$ is a normalized all-ones vector and has the same value for all $s$.

The Laplacian representation is known to be able to capture the geometry of the underlying state space~\cite{mahadevan2005proto, machado2017laplacian}, and thus has been applied in option discovery~\cite{machado2017laplacian, jinnai2019exploration} and reward shaping~\cite{wu2018laplacian}.

In the Laplacian framework for option discovery~\cite{machado2017laplacian}, each dimension of the Laplacian representation defines an intrinsic reward function $r_i(s, s')=\mathbf{e}_i[s]-\mathbf{e}_i[s']$. The options~\cite{sutton1999between} are discovered by maximizing the cumulative discounted intrinsic reward. These options act at different time scales; that is, when an agent follows an option to take actions, the length of its trajectory until termination varies across different options (see Fig. 3 in~\cite{machado2017laplacian}).
Such a property makes these options helpful for exploration: longer options enable agents to quickly reach distant areas and shorter options ensure sufficient exploration in local areas.

When using the Laplacian representation for reward shaping in goal-achieving tasks~\cite{wu2018laplacian}, the reward is shaped based on Euclidean distance, as in~\cite{pong2018temporal, nachum2018data}.
Specifically, the pseudo-reward is defined as the negative $\ell_2$ distance between the agent's state and the goal state in representation space: $r_t = -\|\varphi(s_{t+1}) - \varphi(s_\text{goal})\|_2$.
Since the Laplacian representation can reflect the geometry of the environment dynamics, such pseudo-reward can be helpful in accelerating the learning process.

\subsection{Approximating Laplacian Representation}
\label{subsec:background:spectral graph drawing}

Obtaining Laplacian representation by directly computing eigendecomposition of graph Laplacian requires access to transition dynamics of the environment and involves expensive matrix operations, which is infeasible for environments with a large or even continuous state space.
One efficient approach for approximating Laplacian representation is proposed by~\citet{wu2018laplacian}, which minimizes the following spectral graph drawing objective~\cite{koren2005drawing}:
\begin{equation}
\begin{aligned}
    \min_{\fone, \cdots, \fd} \quad & \sum_{i=1}^d \fthei^\top L \fthei \\
    \st \quad & {\xbf{u}{i}^\top\xbf{u}{j} = \delta_{ij}}, \forall i, j = 1,\cdots, d,
\end{aligned}
\label{eqn:graph drawing}
\end{equation}
where $(\mathbf{u}_1, \cdots, \mathbf{u}_d) \in\mathbb{R}^{|\mathcal{S}|\times d}$ are to approximate the eigenvectors $(\eone,\cdots,\ed)$, and $\delta_{ij}$ is the Kronecker delta.
However, minimizing such an objective can only ensure that $\mathbf{u}_1, \cdots, \mathbf{u}_d$ span the same subspace as $\eone,\cdots,\ed$, as mentioned in~\cite{wu2018laplacian}. It does not guarantee $\mathbf{u}_i=\mathbf{e}_i$ for $i=1,\cdots,d$, because the global minimizer $(\eone,\cdots,\ed)$ is not unique.

Transforming $(\eone,\cdots,\ed)$ with an arbitrary orthogonal transformation also achieves global minimum~\cite{koren2005drawing}. 
Therefore, the problem in Eqn.~(\ref{eqn:graph drawing}) does not ensure that the solution is the eigenvectors, and may converge to any other global minima. 
Accordingly, the learned Laplacian representations may diverge from the ground truth.
We will show in Sec.~\ref{sec:exp} that such representation is less helpful in discovering exploratory options and reward shaping.

\section{Method}
\label{sec:method}
As discussed in Sec.~\ref{subsec:background:spectral graph drawing}, the graph drawing objective is invariant under orthogonal transformation: applying an arbitrary orthogonal transformation to the $d$ smallest eigenvectors also yields a global minimizer, 
which hinders learning Laplacian representations close to the ground truth.

We then consider breaking such invariance for achieving more accurate approximation.
To this end, we reformulate the graph drawing objective into a weighted-sum form, yielding the following generalized graph drawing objective:
\begin{equation}
\label{eqn:ggd}
    \begin{aligned}
    \min_{\fone,\cdots,\fd}\quad &{\sum_{i=1}^d c_i \fthei^\top L \fthei} \\
    \st \quad & {\xbf{u}{i}^\top\xbf{u}{j} = \delta_{ij}}, \forall i, j = 1,\cdots, d,
    \end{aligned}
\end{equation}
where $c_i$ is the coefficient for the $i$-th term $\fthei^\top L \fthei$. When $c_i=1$ for every $i$, it degenerates to the original graph drawing objective in Eqn.~\eqref{eqn:graph drawing}.

We find that, under mild assumptions, if $c_1,\cdots,c_d$ are strictly decreasing, then the $d$ smallest eigenvectors (ground-truth Laplacian representation) make the \textit{unique} global minimizer of the above generalized graph drawing objective, as stated in the following theorem.
\begin{theorem}
\label{thm:optimiality and uniqueness}
Assume $\fthei \in \mathrm{span}(\{\eone,\cdots,\ed\})$, $\forall i=1,\cdots,d$, and $\lambda_1 < \cdots <\lambda_d$. Then, $c_1 > \cdots > c_d > 0$ is a sufficient condition for the generalized graph drawing objective to have a unique global minimizer $(\fone^*,\cdots,\fd^*)=(\eone,\cdots,\ed)$, and the corresponding minimum is $\sum_{i=1}^dc_i\lambda_i$.
\end{theorem}
\begin{proof}
Here we give the proof sketch of Theorem~\ref{thm:optimiality and uniqueness} and the full proof is deferred to the Appendix. Denote the objective of problem~\eqref{eqn:ggd} as $h(\fone,\cdots,\fd)$. Let $g=h(\fone,\cdots,\fd)-\sum_{i=1}^dc_i\lambda_i$. 
Since $\fthei^\top\xbf{u}{j}=\delta_{ij}$ and $\fthei \in \mathrm{span}(\{\eone,\cdots,\ed\})$, without loss of generality, $(\fone,\cdots,\fd)$ can be written as $(\eone,\cdots,\ed)Q$, where $Q=(q_{ij})$ is an orthogonal matrix.

We first prove optimality. By applying Fubini's Theorem~\cite{fubini1907sugli} to $g$, we can rewrite $g$ as $g=\sum_{k=1}^{d-1}s_k(\lambda_{k+1}-\lambda_{k})$, where $s_k$ depends on $Q$ and $d$. We can prove $\forall k \in \{1,\cdots,d-1\}, s_k \geqslant 0$ (this is given by $c_1 > \cdots > c_d > 0$). Hence $g \geqslant 0$. Notice that the inequality is tight when $(\fone,\cdots,\fd)=(\eone,\cdots,\ed)$, which proves the optimality.

Then, we prove uniqueness by contradiction. Denote the global minimum $\sum_{i=1}^dc_i\lambda_i$ as $p^*$. Assume that there exists another global minimizer, denoted as $(\fonetilde,\cdots,\fdtilde)$, \ie, $\exists i, \ftheitilde \neq \pm\ethei$. Here we require $\ftheitilde \neq \pm\ethei$ because the sign of $\ethei$ is arbitrary. Again, we rewrite $(\fonetilde,\cdots,\fdtilde)=(\eone,\cdots,\ed)\Qtilde$, where $\Qtilde=(\qtilde_{ij})$ is an orthogonal matrix. Therefore, the assumption is equivalent to $\exists i \in \{1,\cdots,d\}, \qtilde_{ii} \notin \{1,-1\}$.  Due to the optimality of $(\fonetilde,\cdots,\fdtilde)$, we know $h(\fonetilde,\cdots,\fdtilde)=p^*$. However, we can prove that this equality implies $\forall i \in \{1,\cdots,d\}, \qtilde_{ii} \in \{1,-1\}$. This contradicts with $\exists i, \ftheitilde \neq \pm\ethei$, and hence proves uniqueness.
\end{proof}

We will empirically show that the two assumptions hold in our experiments (see Sec.~\ref{subsubsec:verification of assumptions}). Based on the above theorem, we can choose $c_1 > \cdots > c_d > 0$ to obtain a learning objective that can faithfully approximate the Laplacian representation. A natural choice is $c_d=1, c_{d-1}=2, \cdots, c_1=d$, which gives the following objective:
\begin{equation}
\begin{aligned}
\label{eqn:sgd}
\min_{\fone,\cdots,\fd}\quad &{\sum_{i=1}^d (d-i+1) \fthei^\top L \fthei} \\
    \st \quad & {\xbf{u}{i}^\top\xbf{u}{j} = \delta_{ij}}, \forall i, j = 1,\cdots, d.
\end{aligned}
\end{equation}

We use the above objective throughout the rest of our paper, and conduct ablative experiments with other choices of the coefficients (see Sec.\ref{subsubsec:effectiveness of generalized objectives}).

Note that Theorem~\ref{thm:optimiality and uniqueness} implies a property of the generalized graph drawing objective in Eqn.~\eqref{eqn:ggd}: there is one-to-one correspondence between its solutions and the smallest $d$ eigenvectors, \ie, $\fthei^* = \ethei, \forall i \in \{1,\cdots,d\}$.
With this, a specific dimension (\eg, the 2nd dimension) of the Laplacian representation can be easily derived from the corresponding solution (\eg, $\xbf{u}{2}$). 
This exact correspondence is useful for studying how each dimension of the representation influences an RL task, \eg, reward shaping (see Sec.~\ref{subsec: exp:reward shaping}). 
Note that the spectral graph drawing objective does not have such a property.

The above theoretical results can be easily generalized to the function space (\ie, Hilbert space), which corresponds to a continuous state space in RL (see Appendix).

\textbf{Training objective} \quad In RL applications, it is hard to directly optimize the problem~\eqref{eqn:sgd} because $L$ is not accessible and enumerating the state space may be infeasible. To make the optimization amenable, we follow the practice in~\cite{wu2018laplacian} to express the objective as an expectation. 
The objective in Eqn.~\eqref{eqn:sgd} can be rewritten as
\begin{equation}
\label{eqn:sgd_rewrite}
  \sum_{k=1}^d\sum_{i=1}^k \mathbf{u}_i^\top L \mathbf{u}_i
   = \sum_{k=1}^d\sum_{i=1}^k \sum_{(s,s')\in \mathcal{E}} (\mathbf{u}_i[s] - \mathbf{u}_i[s'])^2
\end{equation}
where the inner summation of the right hand side is over all edges (\ie transitions) in the graph, and $\mathbf{u}_i[s]$ denotes the entry of vector $\mathbf{u}_i$ corresponding to state $s$. 
In practice, we train a neural network $\phi(s)$ with $d$-dimension output $[f_1(s),\cdots,f_d(s)]$ to approximate the Laplacian representation for state $s$. 
Since we only have sampled transitions, we can express Eqn.~\eqref{eqn:sgd_rewrite} as an expectation and minimize the following objective
\begin{equation}
\label{eqn:training obj}
    G(f_1,\cdots,f_d)\, \triangleq \, \mathbb{E}_{(s,s')\sim \mathcal{T}} \sum_{k=1}^d \sum_{i=1}^k \left(f_i(s) - f_i(s')\right)^2
\end{equation}
where $(s, s')$ is a state-transition sampled from a dataset of transitions $\mathcal{T}$.

The orthonormal constraints in Eqn.~\eqref{eqn:sgd} can be implemented as a penalty term:
\begin{equation}
\label{eqn: penalty}
    \mathbb{E}_{s\sim\rho, s'\sim\rho} \sum_{l=1}^d \sum_{j=1}^l \sum_{k=1}^l h_{jk}(s, s')
\end{equation}
where
\begin{equation}
    h_{jk}(s, s') = \left(f_j(s)f_k(s) - \delta_{jk}\right) \left(f_j(s')f_k(s') - \delta_{jk}\right).
\end{equation}
Here $\rho$ denotes the distribution of states in $\mathcal{T}$. Please refer to~\cite{wu2018laplacian} and the Appendix for  detailed derivation.

\section{Experiments}
\label{sec:exp}

\begin{figure}[t]
     \centering
    \includegraphics[width=\linewidth]{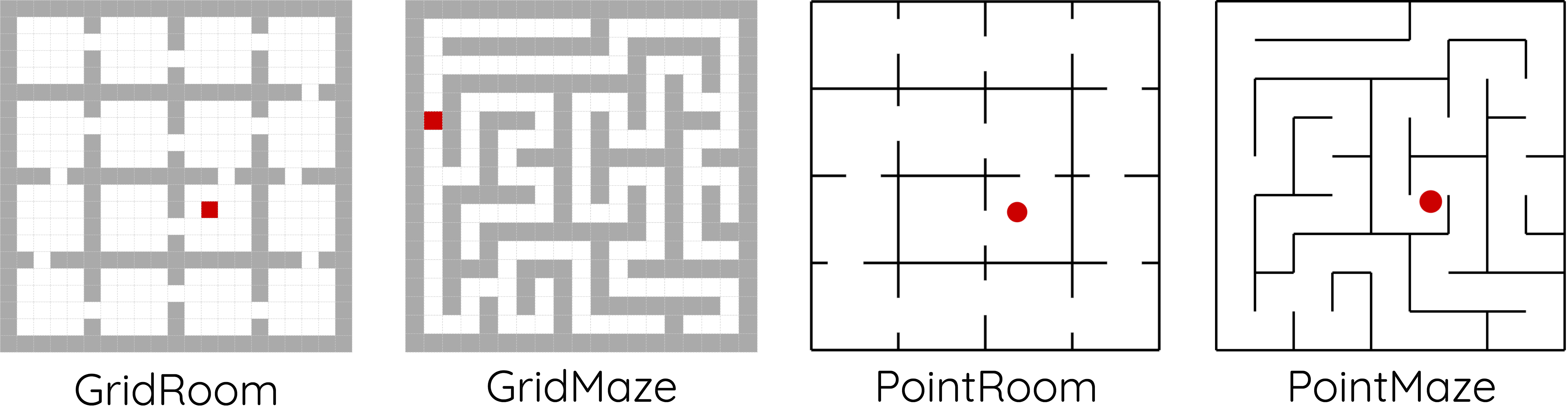}
    \vskip -0.05in
    \caption{Environments for experiments (agents depicted in red).}
     \label{fig:environments}
     \vskip -0.05in
\end{figure}

In this section, we conduct extensive experiments to validate the effectiveness of our method in improving learned Laplacian representation.
Specifically, in Sec.~\ref{subsec: exp:laprep}, we evaluate the learned representations on how well they approximate the ground truth.
In Sec.~\ref{subsec: exp:option} and Sec.~\ref{subsec: exp:reward shaping}, we evaluate the learned representations on their effectiveness in two downstream tasks, \ie for discovering exploratory options and improving reward shaping. 
Finally, in Sec.~\ref{subsec: exp:analysis}, we empirically verify the assumptions used in Theorem~\ref{thm:optimiality and uniqueness} and evaluate other coefficient choices.

We use two discrete gridworld environments and two continuous control environments for our experiments (see Fig.~\ref{fig:environments}), following previous work~\cite{wu2018laplacian}.
The gridworld environments are built with MiniGrid~\cite{gym_minigrid} and the continuous control environments are created with PyBullet~\cite{coumans2019}.
Note that for gridworld environments, our setting is not tabular, since we approximate the Laplacian representation via training neural networks on raw observations (such as $(x, y)$ positions or top-view images) rather than learning a mapping table for all states.
For all experiments, we use $d=10$ for the dimension of the Laplacian representation. More details about training setup can be found in the Appendix. 
For clarity, throughout the experiments we use \emph{baseline} to refer to the method of~\cite{wu2018laplacian}.

\subsection{Learning Laplacian Representations}
\label{subsec: exp:laprep}

\begin{figure*}[t]
     \centering
    \includegraphics[width=\linewidth]{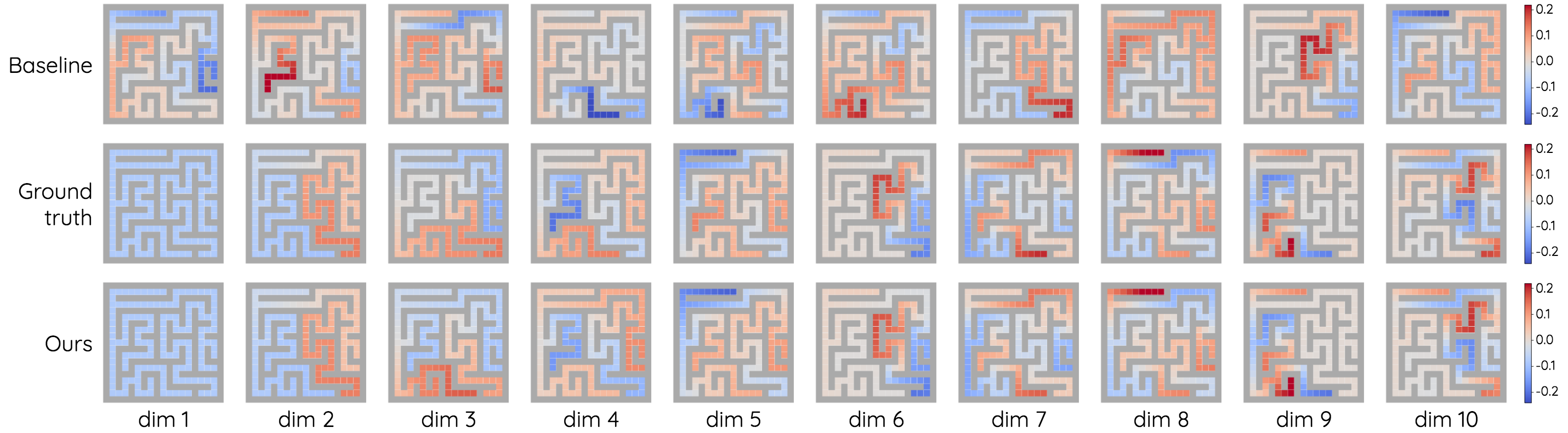}
    \vskip -0.05in
    \caption{Visualization of the learned 10-dimension Laplacian representation and the ground truth on \texttt{GridMaze}.
    Each heatmap shows a dimension of the representation for all states in the environment, where each state is a single cell. Best viewed in color.}
     \label{fig:vis-gridmaze}
\end{figure*}

\begin{figure*}[t]
     \centering
    \includegraphics[width=\linewidth]{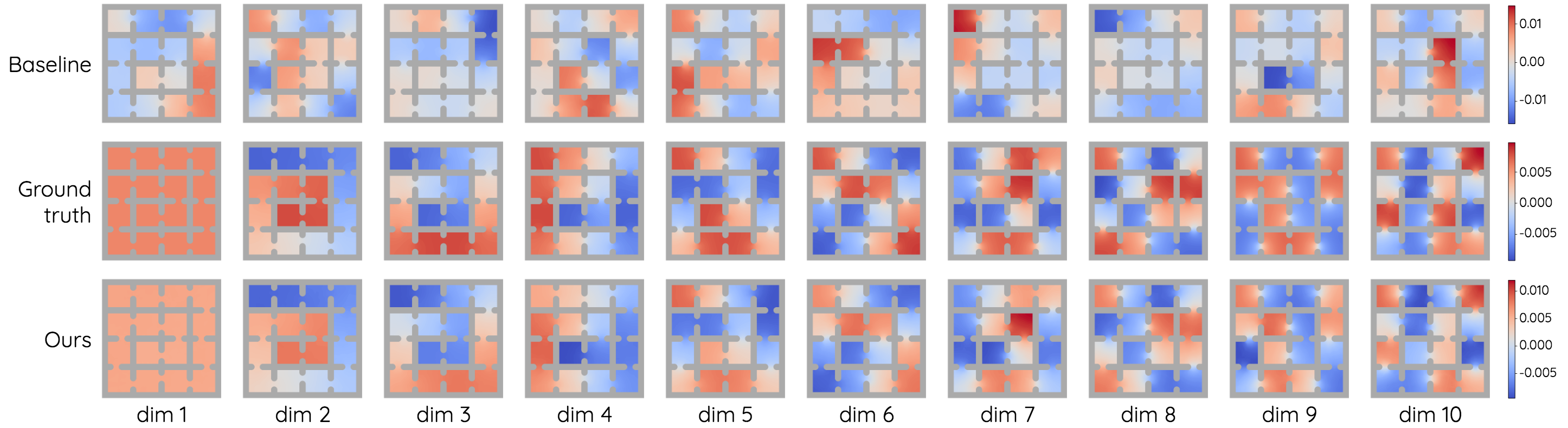}
    \vskip -0.05in
    \caption{Visualization of the learned 10-dimension Laplacian representations and the ground truth on \texttt{PointRoom}.
    Each heatmap shows a dimension of the representation for all states (via interpolation) in the environment. Best viewed in color.}
     \label{fig:vis-pointroom}
     \vskip -0.05in
\end{figure*}

\begin{table*}[t!]
    \centering
    \caption{$\mathrm{SimGT}$ between learned representation and ground truth, averaged across 3 runs. (image) denotes using image observations.}
    \vskip 0.1in
    \begin{tabular}{c c c c c c c}
    \toprule
         Environment & GridRoom & GridMaze & GridRoom (image) & GridMaze (image) & PointRoom & PointMaze \\
    \midrule
    Baseline & 0.239 & 0.220 & 0.310 & 0.229 & 0.239 & 0.255\\
    Ours & 0.991 & 0.962 & 0.985 & 0.984 & 0.963 & 0.779 \\
    \bottomrule
    \end{tabular}
    \label{tab:cosine similarity}
\end{table*}

We take~\cite{wu2018laplacian} as our baseline and following its practice, we also train a neural network to approximate the Laplacian representation, using trajectories collected by a uniformly random policy with random starts. For the environments with discrete state-space (\texttt{GridRoom} and \texttt{GridMaze}), we conduct experiments with both $(x,y)$ position observation and image observation. The ground-truth Laplacian representations (\ie, eigenvectors) are computed by eigendecomposing the graph Laplacian matrix.
For environments with continuous state spaces (\texttt{PointRoom} and \texttt{PointMaze}), we use $(x,y)$ positions as observations, and the ground-truth representations (\ie, eigenfunctions) are approximated by the finite difference method with 5-point stencil~\cite{Knabner2003}. Please see the Appendix for more training details.

To get an intuitive comparison between our method and the baseline in approximating the Laplacian representation, we visualize the learned state representations as well as the ground truth ones of \texttt{GridMaze} and \texttt{PointRoom} in Fig.~\ref{fig:vis-gridmaze} and Fig.~\ref{fig:vis-pointroom}.  As the figures show, our learned Laplacian representations approximate the ground truth much more accurately, while the baseline representations significantly diverge from the ground truth. Similar results in other environments are included in the Appendix.

To quantify the approximation quality of the learned representations, we calculate the absolute dimension-wise cosine similarities between the learned representations and ground-truth ones, and take the average over all dimensions, which yields the following
$\mathrm{SimGT}$ metric:
\begin{equation}
\label{eqn: simGT}
\begin{aligned}
\mathrm{SimGT} = \frac{1}{d}\sum_{i=1}^d \left|\sum_{s}f_i(s)\xbf{e}{i}[s]\right|,
\end{aligned}
\end{equation}
where $s$ is a state, $f_i(s)$ is the $i$-th dimension of the learned representation of $s$ (defined in Sec.~\ref{sec:method}), and $\xbf{e}{i}[s]$ is the $i$-th dimension of the ground truth for $s$ (defined in Sec.~\ref{subsec: background: laprep}) respectively. Note that $\mathrm{SimGT}$ is in the range $[0,1]$, and larger $\mathrm{SimGT}$ means that the representation is closer to the ground truth. As shown in Tab.~\ref{tab:cosine similarity}, our method achieves much higher $\mathrm{SimGT}$ than the baseline, indicating better approximation.

Moreover, we provide empirical evidence for our discussion in Sec.~\ref{sec:method}, that our method converges to the unique global minimizer, while the baseline method can converge to different minima. 
To illustrate this, we visualize the learned representations in 3 different runs, and use the following $\mathrm{SimRUN}$ metric to measure the variance between the representations learned in $l$-th run and those learned in the $m$-th run ($1\leqslant l,m\leqslant3$) via the following:
\begin{equation}
\label{eqn: simRUN}
    \begin{aligned}
    \mathrm{SimRUN}(l,m) = \frac{1}{d}\sum_{i=1}^d\left|\sum_{s}f_i^l(s)f_i^m(s)\right|,
    \end{aligned}
\end{equation}
where $f_i^l(s)$ and $f_i^m(s)$ denote the $i$-th dimension of the learned representation of state $s$ in the $l$-th and $m$-th run, respectively. $\mathrm{SimRUN}(l,m)$ is in the range $[0,1]$, and larger value implies larger inconsistency in the learned representations between 2 runs. As Fig.~\ref{fig:large variance} shows, the learned representations of the baseline method vary a lot across runs, indicating convergence to different minima.
In contrast, our method yields consistent approximations.

The above results demonstrate the superiority of learning Laplacian representation with our proposed objective and empirically support our theoretical analysis in Sec.~\ref{sec:method}.

\begin{figure}[t]
    \includegraphics[width=\linewidth]{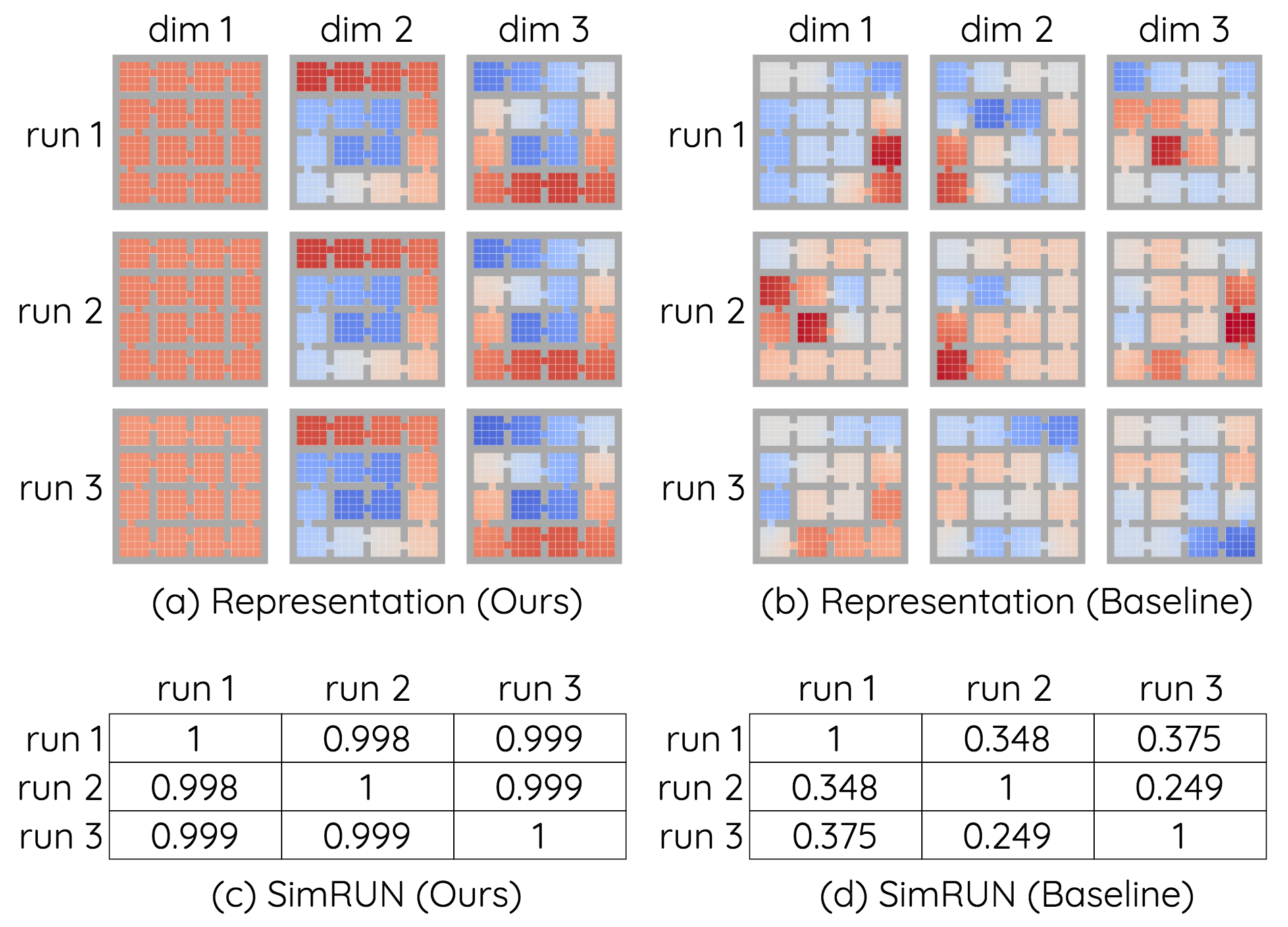}
    \vskip -0.05in
    \caption{(a) and (b): Visualization of first 3 dimensions of representations learned by our method~(a) and baseline method~(b). (c) and (d): $\mathrm{SimRUN}(l,m)$ for $l,m =1,2,3$ computed with our learned representations~(c) and those learned by baseline~(d).}
     \label{fig:large variance}
     \vskip -0.05in
\end{figure}

\subsection{Option Discovery}
\label{subsec: exp:option}

As discussed in Sec.~\ref{subsec: background: laprep}, the Laplacian representation can be applied in discovering exploratory options.
We here evaluate effectiveness of our learned representations in discovering exploratory options, to further show the superiority of our method over the baseline.

Following~\cite{machado2017laplacian}, we learn 2 options for each dimension of the learned representation: one with an intrinsic reward function $r_i(s,s')=f_i(s') - f_i(s)$ and the other with $-r_i(s,s')$, where $f_i(s)$ denotes $i$-th dimension of the representation for state $s$ (see Sec.~\ref{sec:method}).
The options are learned with Deep Q-learning~\cite{mnih2013playing}.
Since the first dimension of the Laplacian representation has the same value for every state (see Sec.~\ref{subsec: background: laprep}), it cannot provide informative intrinsic reward. 
Therefore, we do not learn options for the first dimension of our learned representation.

For each learned option, we compute the average trajectory length for an agent that starts from each state and follows this option until arriving at termination states.
This reflects the time scale at which an option acts: options acting at longer time scales enable agents to quickly reach distant areas, and shorter options ensure sufficient exploration in local areas. 
As Fig.~\ref{fig:option length} shows, trajectory lengths for our method vary across different dimensions, implying the options operate at different time scales.
Such options enable exploration in both nearby and distant areas. 
In contrast, options discovered from the baseline representation operate at similarly short time scales, which may hinder exploration to the distant areas.

To further validate this, we measure the expected number of steps required for an agent (equipped with learned options) to navigate between different rooms in the \texttt{GridRoom} environment.
Specifically, we denote $N_{i\to j}$ as the average number of steps required for an agent starting from room $i$ to reach room $j$. We then calculate $N_{i,j} = (N_{i\to j} + N_{j\to i})/2$ as the expected number of steps required to navigate between two rooms. 
As shown in Fig.~\ref{fig:hitting time}, agents equipped with options discovered from the baseline representation typically take more steps to reach distant rooms.
In comparison, with our method, agents can reach faraway rooms within a similar number of steps as for nearby rooms.

\begin{figure}[t]
     \centering
    \includegraphics[width=\linewidth]{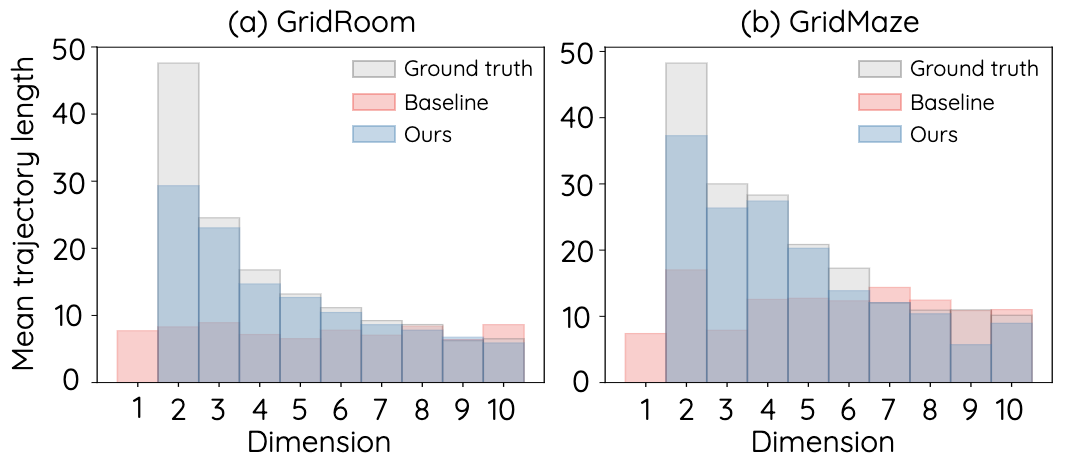}
    \vskip -0.05in
    \caption{Average length of trajectories for options discovered from each dimension of different representations. Longer length implies the option acts at a longer time scale. For each dimension, the result is averaged from corresponding 2 options.}
     \label{fig:option length}
     \vskip -0.1in
\end{figure}

\begin{figure}[t]
\label{fig:hitting time}
     \centering
    \includegraphics[width=\linewidth]{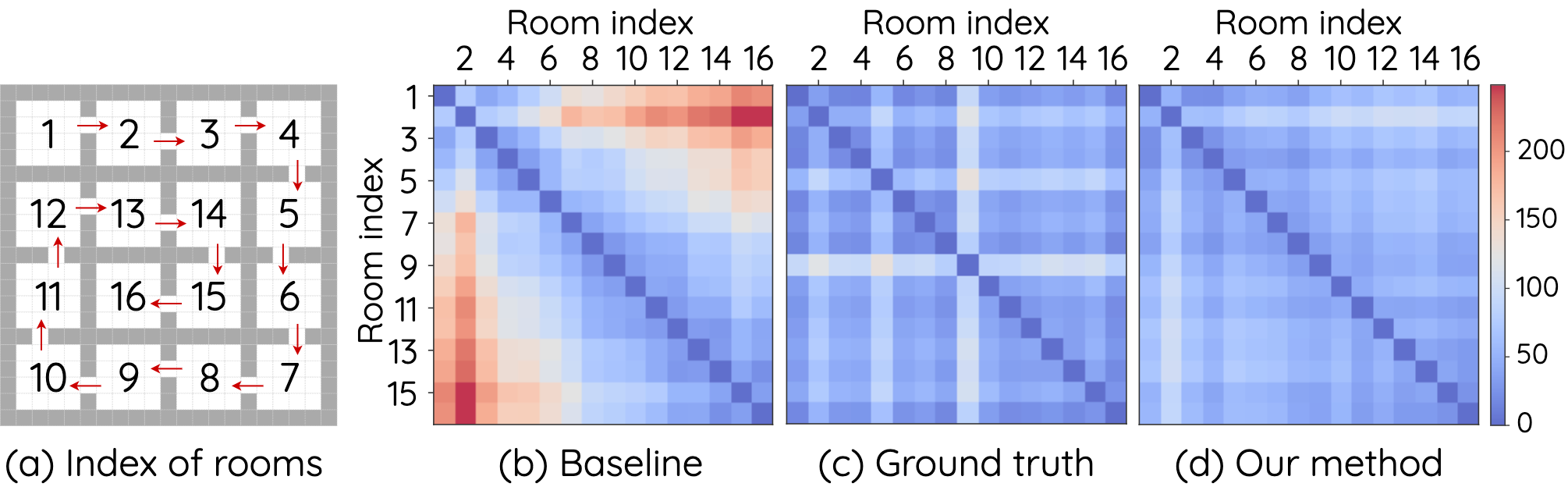}
    \vskip -0.05in
    \caption{(a) Index of rooms. Neighboring rooms are assigned with consecutive indexes. Larger difference between two indexes means that the corresponding rooms are farther in the environment. (b) Average steps needed to navigate between two rooms (baseline). The agent needs many steps to navigate between distant rooms, \eg, room 1 and room 16. (c) Average steps needed to navigate between two rooms (ground truth). (d) Average steps needed to navigate between two rooms (our method).}
    \vskip -0.1in
\end{figure}

\subsection{Reward Shaping}
\label{subsec: exp:reward shaping}

The Laplacian representation can be used for reward shaping
in goal-achieving tasks, as mentioned in Sec.~\ref{subsec: background: laprep}.
Previous work~\cite{wu2018laplacian} uses all $d$ dimensions of the learned representation to define the pseudo-reward, \ie, $r_t = -\|\phi(s_{t+1}) - \phi(s_\text{goal})\|_2$, where $\phi(s)$ is the $d$-dimension representation of state $s$ output by a neural network.
Such pseudo-reward is influenced by each dimension of the representation.
As each dimension of Laplacian representation captures different geometric information about the state space (\eg, see Fig.~\ref{fig:vis-pointroom}), a natural question is: which dimension of the Laplacian representation matters more to reward shaping? Furthermore, can we achieve better reward shaping than using all $d$ dimensions?
\begin{figure}[t]
     \centering
    \includegraphics[width=\linewidth]{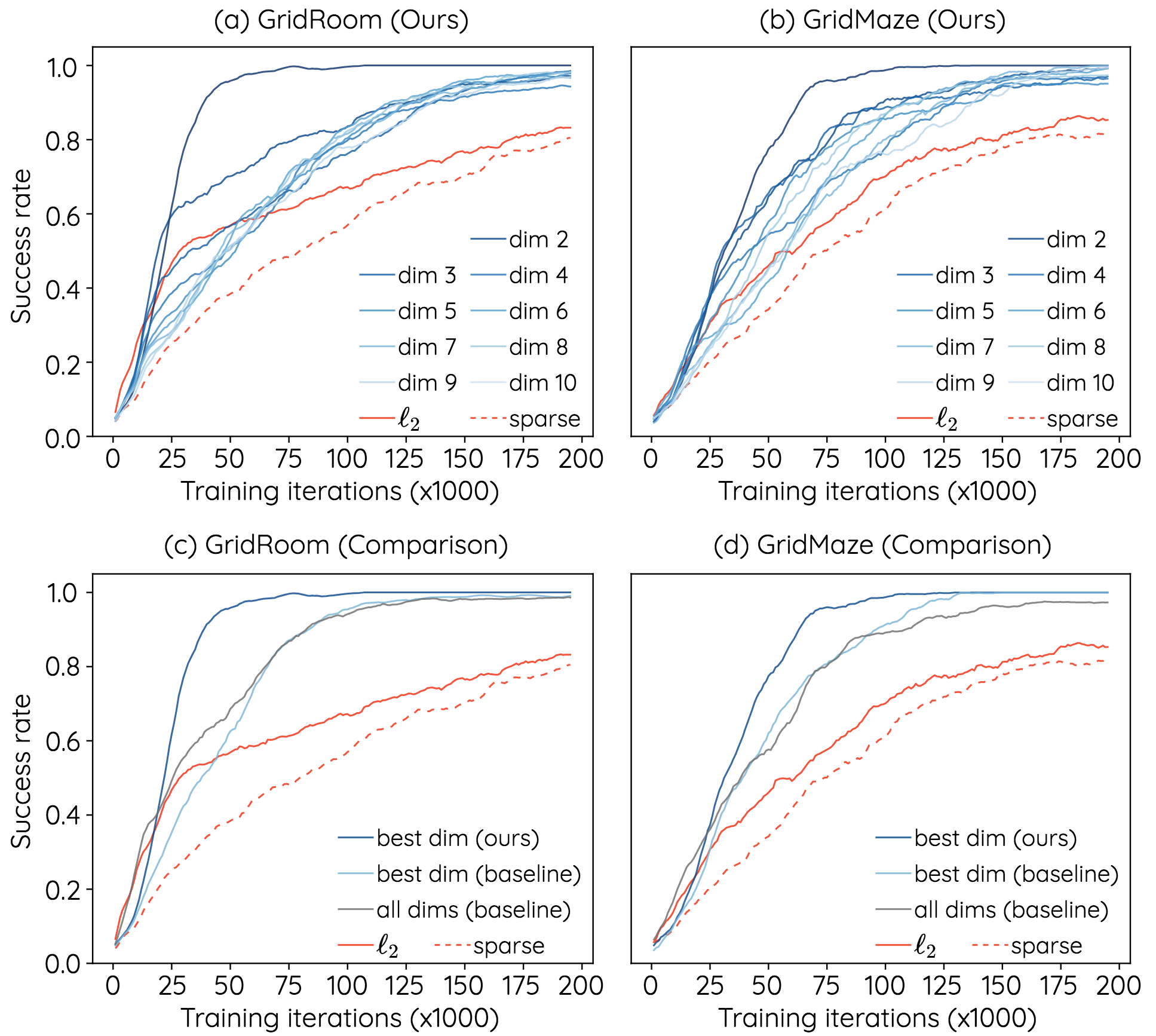}
    \vskip -0.05in
    \caption{Results of reward shaping with each dimension of learned Laplacian representations. $\ell_2$ denotes reward shaping with L2 distance in raw observation space (\ie, $(x,y)$ position), and \emph{sparse} denotes no reward shaping.}
     \label{fig:reward shaping}
     \vskip -0.1in
\end{figure}

In this subsection, we study these questions by comparing individual dimensions of the learned representation for reward shaping. Specifically, we define the pseudo-reward as $r_t = -\left(f_i(s_{t+1}) - f_i(s_\text{goal})\right)^2$ where $f_i(s)$ denotes $i$-th dimension of the representation for state $s$.
Following~\cite{wu2018laplacian}, we train the agents using Deep Q-learning~\cite{mnih2013playing} with $(x, y)$ positions as observations, and measure the agent's success rate of reaching the goal state. 
To eliminate the bias brought by the goal position, we select multiple goals for each environment such that they spread over the state space, and average the results of different goals. 
We do not experiment with the first dimension of our learned representation since every state has same value and the pseudo-reward is always 0.

As shown in Fig.~\ref{fig:reward shaping}(a) and Fig.~\ref{fig:reward shaping}(b), using lower dimensions of our learned Laplacian representation for reward shaping better accelerates the agent's learning process.
Furthermore, we compare using best dimension and all dimensions, of both our learned representation and the baseline representation, in Fig.~\ref{fig:reward shaping}(c) and Fig.~\ref{fig:reward shaping}(d). Results show that the lower dimension of our representation significantly outperforms others, further improving reward shaping.
The above results suggest that lower dimensions of the Laplacian representation are more important to reward shaping.
By learning a high-quality Laplacian representation, our method enables one to more easily choose eigenvectors to use for reward shaping, which leads to improved performance.

\begin{figure}[t]
     \centering
    \includegraphics[width=\linewidth]{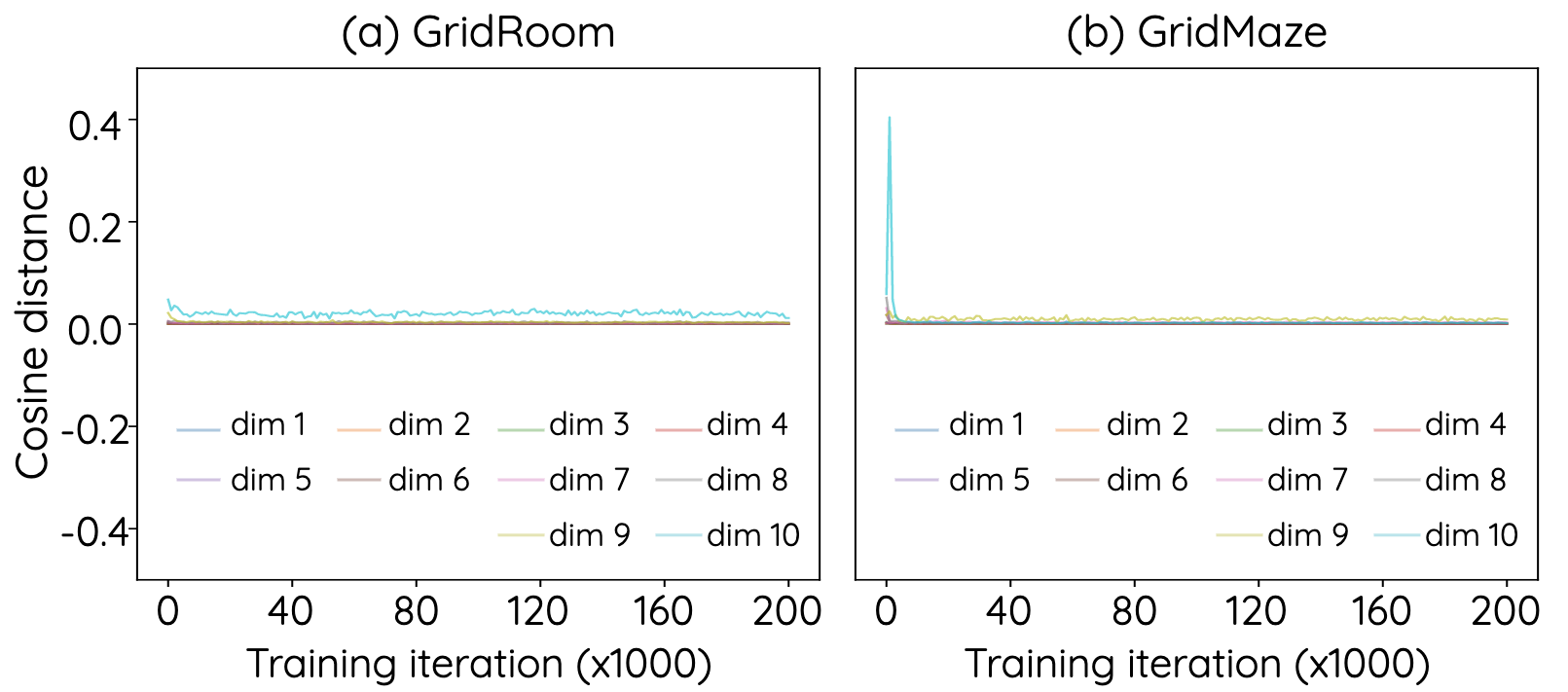}
    \vskip -0.05in
    \caption{Sum of absolute cosine similarities between $\mathbf{u}_i$ and eigenvectors $\mathbf{e}_1, \cdots,\mathbf{e}_d$, during training.}
     \label{fig:lie in span}
     \vskip -0.05in
\end{figure}

\begin{figure}[t]
     \centering
    \includegraphics[width=\linewidth]{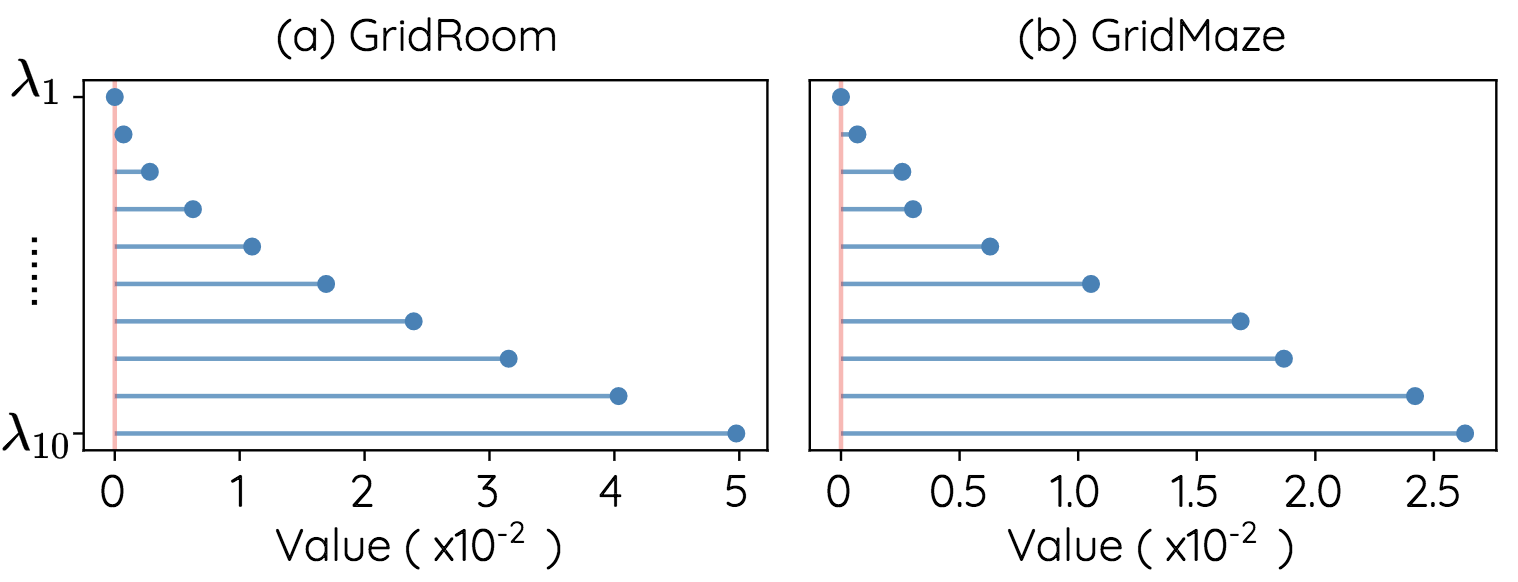}
    \vskip -0.05in
    \caption{Eigenvalues of the graph Laplacian matrix for two environments with discrete state-space.}
     \label{fig:distinct eigvals}
     \vskip -0.05in
\end{figure}

\subsection{Analysis}
\label{subsec: exp:analysis}
Here we first empirically verify the two assumptions in Theorem~\ref{thm:optimiality and uniqueness}, and then conduct ablative experiments with different choices of coefficients $c_1,\cdots,c_d$ for our generalized graph drawing objective~\eqref{eqn:ggd}. 
We use \texttt{GridRoom} and \texttt{GridMaze} environments for experiments in this subsection.

\subsubsection{Verification of assumptions}
\label{subsubsec:verification of assumptions}
The first assumption requires that all optimization variables $\fthei$ lie in the span space of the $d$ smallest eigenvectors during the optimization process, \ie, $\forall i \in \mathrm{span}(\{\eone,\cdots,\ed\})$.

We empirically verify a necessary and sufficient condition for this assumption: for each $\fthei$, the angle between it and its projection onto $\mathrm{span}(\{\eone,\cdots,\ed\})$ (denoted as $\hat{\mathbf{u}}_i \triangleq \sum_{j=1}^{d}(\mathbf{e}_j^\top\mathbf{u}_i)\mathbf{e}_j$) is 0. Specifically, we compute the cosine distance between $\fthei$ and $\hat{\mathbf{u}}_i$ during training. As Fig.~\ref{fig:lie in span} shows, the cosine distance is close to 0 during the whole training process, which implies that the assumption holds in our experiments.

To verify whether the second assumption holds, \ie whether the $d$ smallest eigenvalues of the graph Laplacian are distinct: $\lambda_1 > \cdots > \lambda_d$, we calculate the smallest $d$ eigenvalues of the graph Laplacian of our environments, and plot them in Fig.~\ref{fig:distinct eigvals}.
It is clear that the eigenvalues are distinct, demonstrating the validity of this assumption.

\subsubsection{Evaluation on other coefficient choices}
\label{subsubsec:effectiveness of generalized objectives}
In this above experiments, we choose the coefficients of our generalized graph drawing objective to be $c_i=(d-i+1)$. In this subsubsection, we evaluate the effectiveness of other choices of $c_1,\cdots,c_d$.

Specifically, we select two groups of $c_i$ that are different from the default group (\ie, $c_i=(d-i+1)$ as used in Eqn.~\ref{eqn:sgd}), of which the first group has increasing first-order difference (\ie, $c_{i+1} - c_{i} > c_{i} - c_{i-1}$), while the second group has decreasing first-order difference (\ie, $c_{i+1} - c_{i} < c_{i} - c_{i-1}$). We plot the two groups in Fig.~\ref{fig:coefficients}(a). For comparison, we also include the default coefficient group, which has a constant first-order difference of 1. We then train neural networks with the generalized objective in Eqn.~\eqref{eqn:ggd} using group~1 and group~2 (other experimental settings are the same with Sec.~\ref{subsec: exp:laprep}). We use the similarity between learned representations and ground truth to evaluate the quality of the representations, as done in Sec.~\ref{subsec: exp:laprep}. As can be seen from Fig.~\ref{fig:coefficients}(b), representations learned with objectives using group~1 or group~2 are as good as those learned by the default setting. 
\begin{figure}[t]
     \centering
    \includegraphics[width=0.95\linewidth]{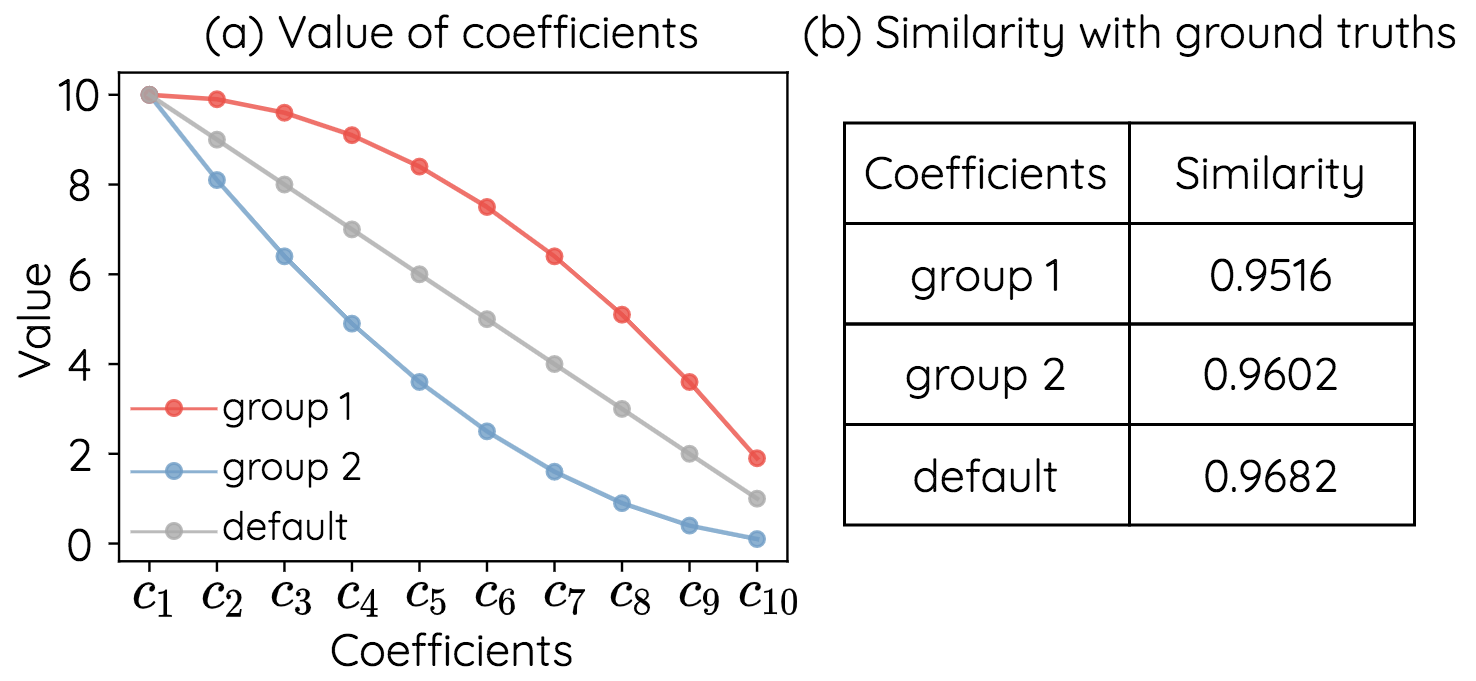}
    \vskip -0.05in
    \caption{(a) Coefficient values of different groups. (b) Absolute cosine similarity (averaged across dimensions) between our learned representation and ground truth.}
     \label{fig:coefficients}
     \vskip -0.2in
\end{figure}

\section{Related Works}
\label{sec:related works}

By viewing the state transition process in RL as a graph where nodes are states and edges are transitions, previous works build a Laplacian-based state representation and successfully apply it in value function approximation~\cite{mahadevan2005proto}, option discovery~\cite{machado2017laplacian} and reward shaping~\cite{wu2018laplacian}.

\citet{mahadevan2005proto} proposes proto-value functions, viewing the Laplacian representations as basis state representations, and use them to approximate value functions. Recently, \citet{machado2017laplacian} introduce a framework for option discovery, which builds options from the Laplacian representation. They show that such options are useful for multiple tasks and helpful in exploration. Later, \citet{machado2018eigenoption} extend their Laplacian option discovery framework to settings where handcrafted features are not available, based on an connection between proto-value functions~\cite{mahadevan2005proto} and successor representations~\cite{dayan1993improving, stachenfeld2014design, barreto2016successor}. \citet{jinnai2019exploration} leverage the approximated Laplacian representation to learn deep covering options for exploration.

Our work focuses on better approximating the Laplacian representation in environments with large state space. Most related to our method is the approach proposed in~\cite{wu2018laplacian}. The authors optimize a spectral graph drawing objective~\cite{koren2005drawing} to approximate the eigenvectors. Though being efficient, their method has difficulties in learning a Laplacian representation close to the ground truth, due to the fact that their minimization objective has infinitely many other global minimizers besides the eigenvectors. Our work improves their method by proposing a new objective, which admits eigenvectors as its unique global minimizer and hence greatly increases the approximation quality in empirical evaluations. Other approaches for approximating the Laplacian representation include performing singular value decomposition on the incidence matrix~\cite{machado2017laplacian, machado2018eigenoption}, training neural networks with constrained stochastic optimization~\cite{shaham2018spectralnet} or bi-level stochastic optimization~\cite{pfau2018spectral}. However, as discussed in~\cite{wu2018laplacian}, these approaches either require expensive matrix operations or suffer poor scalability.

\section{Conclusion}
\label{sec:conclusion}

Laplacian representation provides a succinct and informative state representation for RL, which captures the geometry of the underlying state space. Such representation is beneficial in discovering exploratory options and reward shaping. In this paper, we propose a new objective that greatly improves the approximation quality of the learned Laplacian, for environments with large or even continuous state space. We demonstrate the superiority of our method over previous work via theoretical analysis and empirical evaluation. Our method is efficient and simple to implement. With our method, one can learn high-quality Laplacian representation and apply it to various RL tasks such as option discovery and reward shaping.

\section{Acknowledgements}

Jiashi Feng is supported by the National Research Foundation, Singapore under its AI Singapore Programme (AISG Award No: AISG-100E-2019-035), Singapore National Research Foundation (“CogniVision” grant NRF-CRP20-2017-0003).

\bibliography{example_paper}
\bibliographystyle{icml2021}

\clearpage

\section*{A. Proof of Theorem 1}
To prove Theorem~\ref{thm:optimiality and uniqueness}, we first introduce the following Lemma~\ref{lm:equal block}.
\begin{lemma}
\label{lm:equal block}
Let $Q=(q_{ij}) \in \mathbb{R}^{d\times d}$ be an orthogonal matrix, and $A=(a_{ij})=Q\odot Q=(q^2_{ij})$. For $k\in\{1,\cdots,d-1\}$, we have 
\begin{equation}
    \begin{aligned}
    \sum_{i=1}^k\sum_{j=k+1}^{d}a_{ij} =
    \sum_{i=1}^k\sum_{j=k+1}^{d}a_{ji}
    \end{aligned}
\end{equation}
\end{lemma}
\begin{proof}
Since $Q$ is an orthogonal matrix, we know that the sum of $A$'s first $k$ rows is equal to that of $A$'s first $k$ columns, \ie:
\begin{equation}
    \begin{aligned}
    \sum_{i=1}^{k}\sum_{j=1}^{d}a_{ij} = \sum_{i=1}^{k}\sum_{j=1}^{d}a_{ji} = k
    \end{aligned}
\end{equation}
Therefore, we have
\begin{equation}
    \begin{aligned}
    \sum_{i=1}^k\sum_{j=k+1}^d a_{ij} = & \sum_{i=1}^k\sum_{j=1}^d a_{ij} - \sum_{i=1}^k\sum_{j=1}^k a_{ij} \\
    = & \sum_{i=1}^k\sum_{j=1}^d a_{ji} - \sum_{i=1}^k\sum_{j=1}^k a_{ij} \\
    = & \sum_{i=1}^k\sum_{j=1}^d a_{ji} - \sum_{i=1}^k\sum_{j=1}^k a_{ji} \\
    = & \sum_{i=1}^k\sum_{j=k+1}^{d}a_{ji}
    \end{aligned}
\end{equation}
\end{proof}
If we view $A$ as a block matrix, \ie
\begin{equation}
A = 
    \begin{pmatrix}
        A_{1:k, 1:k} & A_{1:k, k+1:d}\\
        A_{k+1:d, 1:k} & A_{k+1:d, k+1:d}
    \end{pmatrix},
\end{equation}
Lemma~\ref{lm:equal block} says the sum of elements in $A_{1:k, k+1:d}$ is equal to the sum of elements in $A_{k+1:d, 1:k}$.

With Lemma~\ref{lm:equal block}, now we can prove Theorem~\ref{thm:optimiality and uniqueness} as following.
\begin{proof}
Let $E$ denote the matrix of the first $d$ eigenvectors of $d$, \ie, $E=(\eone, \cdots, \ed)$. Since $\fthei \in \mathrm{span}(\{\eone, \cdots, \ed\})$, and $\fthei^\top \xbf{u}{j}=\delta_{ij}$, we may rewrite $(\fone,\cdots,\fd)=EQ$, where $Q=(q_{ij})\in \mathbb{R}^{d\times d}$ is an orthogonal matrix. Let $A=(a_{ij})=Q\odot Q=(q_{ij}^2)$. Then, the objective of problem~\eqref{eqn:ggd} becomes:

\begin{equation}
\label{eqn:rewrite objective}
    \begin{aligned}
    h(\fone,\cdots,\fd) \triangleq & \sum_{i=1}^{d}c_i\fthei^\top L \fthei \\
    = & \sum_{i=1}^d c_i (\sum_{j=1}^d q_{ji}\xbf{e}{j})^\top L (\sum_{j=1}^d q_{ji}\xbf{e}{j}) \\ 
    = & \sum_{i=1}^d c_i \sum_{j=1}^d q_{ji}^2 \xbf{e}{j}^\top L \xbf{e}{j} \\
    = &\sum_{i=1}^d\sum_{j=1}^d c_i a_{ji}\lambda_j
    \end{aligned}
\end{equation}
We first prove optimality. Let $g$ denote the gap between the objective and $\sum_{i=1}^d c_i \lambda_i$. We have
\begin{equation}
\label{eqn:gap wrt lambda 1}
    \begin{aligned}
    g \triangleq & \sum_{i=1}^{d}c_i\fthei^\top L \fthei - \sum_{i=1}^dc_i\lambda_i \\
    = &\sum_{i=1}^dc_i\sum_{j=1}^{d}a_{ji}\lambda_j - \sum_{i=1}^dc_i\lambda_i\\
    \end{aligned}
\end{equation}
Note that $\sum_{j=1}^{d}a_{ji}=1$, then we have:
\begin{equation}
\label{eqn:gap wrt lambda 2}
    \begin{aligned}
    g = &\sum_{i=1}^dc_i\sum_{j=1}^{d}a_{ji}\lambda_j - \sum_{i=1}^dc_i\sum_{j=1}^{d}a_{ji}\lambda_i\\
    = & \sum_{i=1}^dc_i\sum_{j=1}^{d}a_{ji}(\lambda_j-\lambda_i) \\
    = & \sum_{i=1}^d\sum_{j=1}^{d}c_ia_{ji}(\lambda_j-\lambda_i) 
    \end{aligned}
\end{equation}
Let $\Delta_{ji} = \lambda_j - \lambda_i$, and $r_{ji}=c_ia_{ji}$,  then we can rewrite $g$ as:
\begin{equation}
\label{eqn:gap wrt delta 1}
    \begin{aligned}
    g = & \sum_{i=1}^d\sum_{j=1}^d r_{ji}\Delta_{ji}
    \end{aligned}
\end{equation}
Note that $\Delta_{ii}=0$ and that, for $j \geqslant i$, $\Delta_{ji} = \Delta_{i+1, i} + \Delta_{i+2, i+1}+\cdots+\Delta_{j-1, j-2}+\Delta_{j, j-1} = \sum_{k=i}^{j-1}\Delta_{k+1,k}$. We then apply Fubini's Theorem~\cite{fubini1907sugli} to $g$:
\begin{equation}
\label{eqn:gap wrt delta 2}
    \begin{aligned}
    g = & \sum_{j\geqslant i} r_{ji}\Delta_{ji} + \sum_{j\leqslant i} r_{ji}\Delta_{ji}\\
    = & \sum_{j\geqslant i}(r_{ji}- r_{ij})\Delta_{ji} \\ 
    = & \sum_{j>i}(r_{ji}-r_{ij})\sum_{k=i}^{j-1}\Delta_{k+1,k}\\
    = & \sum_{j>k\geqslant i}(r_{ji}-r_{ij})\Delta_{k+1,k} \\
    = & \sum_{k=1}^{d-1}(\sum_{i=1}^k\sum_{j=k+1}^d(r_{ji}-r_{ij}))\Delta_{k+1,k} \\
    \triangleq & \sum_{k=1}^{d-1}s_{k}\Delta_{k+1,k} 
    \end{aligned}
\end{equation}

Note that for $s_k$, we have
\begin{equation}
\label{eqn:sk term}
\begin{aligned}
s_k = &\sum_{i=1}^k\sum_{j=k+1}^d(r_{ji}-r_{ij}) \\
=& \sum_{i=1}^k\sum_{j=k+1}^d (c_ia_{ji}-c_j a_{ij})\\
=&\sum_{i=1}^kc_i\sum_{j=k+1}^da_{ji}-\sum_{j=k+1}^d c_j \sum_{i=1}^ka_{ij} \\
\geqslant & c_k\sum_{i=1}^k\sum_{j=k+1}^d a_{ji} - c_{k+1}\sum_{j=k+1}^d\sum_{i=1}^k a_{ij}
\end{aligned}
\end{equation}
According to Lemma~\ref{lm:equal block}, we know
\begin{equation}
\label{eqn:lemma says this}
    \begin{aligned}
    \sum_{i=1}^k\sum_{j=k+1}^d a_{ji} = \sum_{j=k+1}^d\sum_{i=1}^k a_{ij}
    \end{aligned}
\end{equation}
Therefore, we have
\begin{equation}
\label{eqn:sk larger than zero}
    \begin{aligned}
    s_k \geqslant (c_k - c_{k+1})\sum_{i=1}^k\sum_{j=k+1}^da_{ji} \geqslant 0.
    \end{aligned}
\end{equation}

Since $\Delta_{k+1,k} > 0$, with Eqn.~\eqref{eqn:sk larger than zero}, we can obtain
\begin{equation}
\label{eqn:gap larger than zero}
    \begin{aligned}
    g = &\sum_{i=1}^{d}c_i\fthei^\top L \fthei - \sum_{i=1}^dc_i\lambda_i \\
    = &\sum_{k=1}^{d-1}s_{k}\Delta_{k+1,k} \\
    \geqslant & 0.
    \end{aligned}
\end{equation}

\Ie, the following inequality holds:
\begin{equation}
\label{eqn:optimality}
    \begin{aligned}
    \sum_{i=1}^{d}c_i\fthei^\top L \fthei \geqslant \sum_{i=1}^dc_i\lambda_i
    \end{aligned}
\end{equation}

Since $\ethei^\top L \ethei=\lambda_i$, the inequality is tight when
\begin{equation}
\label{eqn:optmial solution}
    \begin{aligned}
    (\fone,\cdots,\fd) = (\eone,\cdots,\ed).
    \end{aligned}
\end{equation}
Therefore, we conclude that $\sum_{i=1}^dc_i\lambda_i$ is the global minimum, and $(\eone,\cdots,\ed)$ is one minimizer.

Next, we prove uniqueness. Assume that there is another minimizer for this problem, denoted as $(\fonetilde,\cdots,\fdtilde)$. We have
\begin{equation}
\label{prop: not all equal}
    \begin{aligned}
    & (\fonetilde,\cdots,\fdtilde)\neq(\eone,\cdots,\ed) \\
    \Leftrightarrow\quad &\exists i\in\{1,\cdots,d\}, \ftheitilde \neq \pm\ethei
    \end{aligned}
\end{equation}
Here we require $\ftheitilde \neq \pm \ethei$ because the sign of $\ethei$ is arbitrary and hence we do not distinguish them.
Again, $(\fonetilde,\cdots,\fdtilde)$ can be written as $(\eone,\cdots,\ed)\Qtilde$, where $\Qtilde=(\qtilde_{ij})\in \mathbb{R}^{d\times d}$ is an orthogonal matrix. Therefore, proposition in Eqn.~\eqref{prop: not all equal} is equivalent to
\begin{equation}
\label{eqn:not one or minus one}
    \begin{aligned}
    \exists i \in \{1,\cdots,d\}, \tilde{q}_{ii} \notin \{1, -1\}.
    \end{aligned}
\end{equation}
Denote $\Atilde=(\atilde_{ij})=\Qtilde\odot\Qtilde=(\qtilde^2_{ij})$. By the optimality of $(\fonetilde,\cdots,\fdtilde)$, we have 
\begin{equation}
\label{eqn:gap between optimal solution (eq)}
    \begin{aligned}
    \sum_{i=1}^dc_i\ftheitilde^\top L \ftheitilde - \sum_{i=1}^dc_i\lambda_i = 0\\
    \end{aligned}
\end{equation}
From Eqn.~\eqref{eqn:gap wrt lambda 1} to Eqn.~\eqref{eqn:sk larger than zero}, we know
\begin{equation}
\label{eqn:gap between optimal solution (ineq)}
    \begin{aligned}
     \sum_{i=1}^dc_i\ftheitilde^\top L \ftheitilde - \sum_{i=1}^dc_i\lambda_i \geqslant  0 
    \end{aligned}
\end{equation}
The equality holds if and only if $\atilde_{ji}=0, \forall (i, j) \in \{(i,j)|\  j > i\} $.
Additionally, according to Lemma~\ref{lm:equal block}, we have 
\begin{equation}
\label{eqn:lemma says this (tilde)}
    \begin{aligned}
    \sum_{i=1}^k\sum_{j=k+1}^d\atilde_{ji} = \sum_{i=1}^k\sum_{j=k+1}^d\atilde_{ij}
    \end{aligned}
\end{equation}
Therefore, we also have $\atilde_{ji}=0, \forall (i, j) \in \{(i,j)|\  j < i\} $. Accordingly, all off-diagonal elements of $\Atilde$ are 0, \ie, $\atilde_{ij}=0, \forall i\neq j$.
Moreover, since $\Qtilde$ is orthogonal, the following equality holds
\begin{equation}
    \begin{aligned}
    \sum_{j=1}^d\atilde_{ij} =\sum_{j=1}^d\qtilde^2_{ij} = 1, \forall i \in \{1,\cdots,d\}.
    \end{aligned}
\end{equation}
So we have 
\begin{equation}
\label{eqn:all diagonal are one}
    \begin{aligned}
    &\forall i \in \{1,\cdots,d\}, \atilde_{ii} = 1, \\
    \Leftrightarrow\quad&\forall i \in \{1,\cdots,d\}, \qtilde_{ii}\in \{1, -1\}
    \end{aligned}
\end{equation}
which contradicts with proposition in Eqn.~\eqref{eqn:not one or minus one}. Based on the above, we conclude that $(\eone,\cdots,\ed)$ is the unique global miminizer.
\end{proof}

\section*{B. Extension to Continuous setting}
In Sec.~\ref{sec:background} and Sec.~\ref{sec:method}, we discuss the Laplacian representation and our proposed objective in discrete case. In this section we extend previous discussions to continuous settings. Consider a graph with infinitely
many nodes (\ie, states), where weighted edges represent pairwise non-negative affinities (denoted by $D(u, v) \ge 0$ for nodes $u$ and $v$).

Following~\cite{wu2018laplacian}, we give the following definitions. A Hilbert space $\mathcal{H}$ is defined to be the set of square-integrable real-valued functions on graph nodes, \ie $\mathcal{H} = \{ f: \mathcal{S}\to\mathbb{R} \ |\  \int_\mathcal{S} |f(u)|^2 \,d\rho(u) < \infty \}$, associated with the inner-product
\begin{equation}
\label{eqn: inner product of hilbert space}
    \langle f, g \rangle_\mathcal{H} = \int_\mathcal{S} f(u)g(u)\,d\rho(u),
\end{equation}
where $\rho$ is a probability measure, \ie $\int_\mathcal{S} \,d\rho(u)=1$. The norm of a function $f$ is defined as $\langle f, f \rangle_{\mathcal{H}}$. Functions $f, g$ are orthogonal if $\langle f, g \rangle_{\mathcal{H}} = 0$; functions $f_1,\cdots,f_d$ are orthonormal if $\inner{f_i}{f_j} = \delta_{ij}, \forall i, j \in \{1, \cdots,d\}$. The graph Laplacian is defined as a linear operator $\mathscr{L}$ on $\mathcal{H}$, given by
\begin{equation}
\label{eqn: laplacian operator}
    \mathscr{L}f(u) = f(u) - \int_\mathcal{S}f(v)D(u,v)\,d\rho(v).
\end{equation}

Our goal is to learn $f_1, \cdots, f_d$ for approximating the $d$ eigenfunctions $e_1, \cdots, e_d$ associated with the smallest $d$ eigenvalues $\lambda_1, \cdots, \lambda_d$ of $\mathscr{L}$. The graph drawing objective used in~\cite{wu2018laplacian} is
\begin{equation}
\label{eqn: continuous gd}
\begin{aligned}
\min_{f_1,\cdots,f_d} & \quad \sum_{i=1}^d \inner{f_i}{\mathscr{L}f_i} \\
\st & \quad \inner{f_i}{f_j} = \delta_{ij}, \forall i, j= 1,\cdots,d.
\end{aligned}
\end{equation}

Extending this objective to the generalized form gives us
\begin{equation}
\label{eqn: continuous ggd}
\begin{aligned}
\min_{f_1,\cdots,f_d} & \quad \sum_{i=1}^d c_i \inner{f_i}{\mathscr{L}f_i} \\
\st & \quad \inner{f_i}{f_j} = \delta_{ij}, \forall i, j= 1,\cdots,d.
\end{aligned}
\end{equation}
Similarly, for continuous setting, Theorem~\ref{thm:optimiality and uniqueness} can be extended to the following theorem:
\begin{theorem}
\label{thm: continuous optimiality and uniqueness}
Assume $\forall i, f_i \in \mathrm{span}(\{e_1, \cdots, e_d\})$, and $\lambda_1 < \cdots < \lambda_d)$. Then, $c_1 > \cdots > c_d$ is a sufficient condition for the generalized graph drawing objective to have a unique global minimizer $(f_1^*, \cdots, f_d^*) = (e_1, \cdots, e_d)$, and the corresponding minimum is $\sum_{i=1}^d c_i \lambda_i$.
\end{theorem}
To prove the Theorem~\ref{thm: continuous optimiality and uniqueness}, we need the following Lemma~\ref{lm: continuous sum of square} and Lemma~\ref{lm: continuous equal block}.
\begin{lemma}
\label{lm: continuous sum of square}
Let $f_1,\cdots,f_d$ be $d$ orthonormal functions in $\mathrm{span}(\{e_1,\cdots,e_d\})$, and $q_{ji}$ be the inner product of $f_i$ and $e_j$, \ie, $q_{ji} = \inner{f_i}{e_j}, \forall i,j \in \{1,\cdots,d\}$. Then we have (i) $\forall i \in \{1,\cdots,d\}, \sum_{j=1}^dq_{ji}^2=1$, and (ii) $\forall j \in \{1,\cdots,d\}, \sum_{i=1}^dq_{ji}^2=1$.
\end{lemma}
\begin{proof}
First, since $e_1,\cdots,e_d$ form an orthonormal basis, consider projection of $f_i$ onto $e_1, \cdots, e_d$. We have
\begin{equation}
    \begin{aligned}
    f_i = & \sum_{j=1}^d \inner{f_i}{e_j}e_j = \sum_{j=1}^d q_{ji} e_j.
    \end{aligned}
\end{equation}
Since $f_i$ has a norm of 1, we have 
\begin{equation}
    \begin{aligned}
    \inner{f_i}{f_i} = & \inner{\sum_{j=1}^d q_{ji} e_j}{\sum_{j=1}^d q_{ji} e_j} \\
    = & \sum_{j=1}^d q_{ji}^2 = 1.
    \end{aligned}
\end{equation}
The above equation proves (i).
Then, consider projection of $e_j$ onto $f_1, \cdots, f_d$ (note that $f_1, \cdots, f_d$ also form an orthogonal basis for the subspace spanned by $e_1,\cdots,e_d$). We have,
\begin{equation}
    \begin{aligned}
    e_j = \sum_{i=1}^d \inner{e_j}{f_i}f_i = \sum_{i=1}^d q_{ji} f_i.
    \end{aligned}
\end{equation}
Since $e_j$ also has a norm of 1, we have
\begin{equation}
    \begin{aligned}
    \inner{e_j}{e_j} = & \inner{\sum_{i=1}^d q_{ji}f_i}{\sum_{i=1}^d q_{ji}f_i} \\
    = & \sum_{i=1}^d q_{ji}^2 = 1.
    \end{aligned}
\end{equation}
This equation shows that (ii) holds.
\end{proof}
\begin{lemma}
\label{lm: continuous equal block}
Let $f_1,\cdots,f_d$ be $d$ orthonormal functions in $\mathrm{span}(\{e_1,\cdots,e_d\})$, $q_{ji}$ be the inner product of $f_i$ and $e_j$, \ie, $q_{ji} = \inner{f_i}{e_j}, \forall i,j \in \{1,\cdots,d\}$, and $a_{ji}=q_{ji}^2$. Then, for $k\in\{1,\cdots,d-1\}$, we have
\begin{equation}
    \begin{aligned}
     \sum_{i=1}^k\sum_{j=k+1}^d a_{ij} =
     \sum_{i=1}^k\sum_{j=k+1}^d a_{ji}
    \end{aligned}
\end{equation}
\end{lemma}
\begin{proof}
By Lemma~\ref{lm: continuous sum of square}, we have
\begin{equation}
    \begin{aligned}
    \sum_{i=1}^{k}\sum_{j=1}^d a_{ij} = \sum_{i=1}^{k}\sum_{j=1}^d a_{ji} = k
    \end{aligned}
\end{equation}
Therefore, we have
\begin{equation}
    \begin{aligned}
    \sum_{i=1}^k\sum_{j=k+1}^d a_{ij} = & \sum_{i=1}^k\sum_{j=1}^d a_{ij} - \sum_{i=1}^k\sum_{j=1}^k a_{ij} \\
    = & \sum_{i=1}^k\sum_{j=1}^d a_{ji} - \sum_{i=1}^k\sum_{j=1}^k a_{ij} \\
    = & \sum_{i=1}^k\sum_{j=1}^d a_{ji} - \sum_{i=1}^k\sum_{j=1}^k a_{ji} \\
    = & \sum_{i=1}^k\sum_{j=k+1}^{d}a_{ji}
    \end{aligned}
\end{equation}
\end{proof}
With Lemma~\ref{lm: continuous equal block}, we can prove Theorem~\ref{thm: continuous optimiality and uniqueness}.
\begin{proof}
Since $f_i \in \mathrm{span}(\{e_1,\cdots,e_d\})$, without loss of generality, we may rewrite $f_i$ as 
\begin{equation}
    f_i = \sum_{j=1}^d q_{ji} e_j,
\end{equation}
where $q_{ji}=\inner{f_i}{e_j}$. Then, the objective of problem~\eqref{eqn: continuous ggd} is 
\begin{equation}
    \begin{aligned}
    h(f_1,\cdots,f_d) \triangleq &\sum_{i=1}^d c_i \inner{f_i}{\mathscr{L}f_i} \\
    = & \sum_{i=1}^d c_i \inner{\sum_{j=1}^d q_{ji} e_j}{\mathscr{L}\sum_{j=1}^d q_{ji} e_j} \\
    = & \sum_{i=1}^d c_i \inner{\sum_{j=1}^d q_{ji} e_j}{\sum_{j=1}^d q_{ji} \mathscr{L} e_j} \\
    = & \sum_{i=1}^d c_i \inner{\sum_{j=1}^d q_{ji} e_j}{\sum_{j=1}^d q_{ji} \lambda_j e_j} \\
    = & \sum_{i=1}^d c_i \sum_{j=1}^d q_{ji}^2 \lambda_j \inner{e_j}{e_j} \\
    = & \sum_{i=1}^d c_i \sum_{j=1}^d a_{ji}\lambda_j,
    \end{aligned}
\end{equation}
where $a_{ji} = q_{ji}^2$.

Let $g$ denote the gap between the objective and $\sum_{i=1}^d c_i \lambda_i$. Then we have
\begin{equation}
\label{eqn: continuous gap wrt lambda}
    \begin{aligned}
    g \triangleq & \sum_{i=1}^{d}c_i\inner{f_i}{\mathscr{L}f_i} - \sum_{i=1}^dc_i\lambda_i \\
    = &\sum_{i=1}^dc_i\sum_{j=1}^{d}a_{ji}\lambda_j - \sum_{i=1}^dc_i\lambda_i\\
    = & \sum_{i=1}^dc_i\sum_{j=1}^{d}a_{ji}(\lambda_j-\lambda_i) \\
    = & \sum_{i=1}^d\sum_{j=1}^{d}c_ia_{ji}(\lambda_j-\lambda_i) 
    \end{aligned}
\end{equation}
We can see that Eqn.~\eqref{eqn: continuous gap wrt lambda} has the same form as Eqn.~\eqref{eqn:gap wrt lambda 2}. Thus we can follow the same steps as in the proof of Theorem~\ref{thm:optimiality and uniqueness} (\ie, from Eqn.~\eqref{eqn:gap wrt delta 1} to Eqn.~\eqref{eqn:all diagonal are one}, replacing Lemma~\ref{lm:equal block} with Lemma~\ref{lm: continuous equal block}) to finish proving Theorem~\ref{thm: continuous optimiality and uniqueness}.
\end{proof}

\section*{C. Obtaining Training objective}

In~\cite{wu2018laplacian}, the authors express the graph drawing objective as an expectation
\begin{equation}
\label{eqn:training obj for graph drawing}
    \mathbb{E}_{(s,s')\sim \mathcal{T}} \sum_{i=1}^k \left(f_i(s) - f_i(s')\right)^2
\end{equation}
and transform the orthonormal constraints into the following penalty term
\begin{equation}
\label{eqn: penalty graph drawing}
    \mathbb{E}_{s\sim\rho, s'\sim\rho} \sum_{i,j}^k \left(f_i(s)f_j(s) - \delta_{ij}\right) \left(f_i(s')f_j(s') - \delta_{ij}\right).
\end{equation}
Here $k$ denotes the dimension of the representation and $\sum_{i,j}^k$ is short for $\sum_{i=1}^k\sum_{j=1}^k$. From Eqn.~\eqref{eqn:sgd_rewrite}, we can see that our objective can be viewed as the sum of $d$ graph drawing objectives. Thus we can obtain Eqn.~\eqref{eqn:training obj} by summing $d$ objectives in Eqn.~\eqref{eqn:training obj for graph drawing} with $k$ varying from $1$ to $d$. Similarly, we can obtain Eqn.~\eqref{eqn: penalty} by summing $d$ penalty terms in Eqn.~\eqref{eqn: penalty graph drawing}.

\clearpage
\section*{D. Environment Descriptions}

Two discrete gridworld environments used in our experiments: \texttt{GridRoom} and \texttt{GridMaze}, are built with MiniGrid~\cite{gym_minigrid}. The \texttt{GridRoom} environment is a 20$\times$20 grid with 271 states, and the \texttt{GridMaze} environment is a 18$\times$18 grid with 161 states. In both environments, the agent has 4 four actions: moving \textit{left}, \textit{right}, \textit{up} and \textit{down}. When the agent hits the wall, it remains in previous position. Two raw state representations are considered: $(x, y)$ coordinates (scaled within $[-1, 1]$) and top-view image of the grid (scaled within $[0, 1]$).

Two continuous control navigation environments used in our experiments: \texttt{PointRoom} and \texttt{PointMaze}, are built with PyBullet~\cite{coumans2019}. The \texttt{PointRoom} environment is of size 20$\times$20 and each room is of size 5$\times$5. The \texttt{GridMaze} environment is of size 18$\times$18 and the width of each corridor is 2. For both environments, a ball with diameter 1 is controlled to navigate in the environment. It takes a continuous action (within range $[0,2\pi]$) to decide the direction and then move a small step forward along this direction. We consider the $(x, y)$ positions as the raw state representations.

\section*{E. Experiment Configurations}

\subsection*{E.1 Learning Laplacian Representations}

For learning Laplacian representations on \texttt{GridRoom} and \texttt{GridMaze} environments, we collect a dataset of 100,000 transitions using a uniformly random policy with random starts. Each episode has a length of 50. Following~\cite{wu2018laplacian}, we use a fully connected neural network for $(x,y)$ position observations and a convolutional neural network for image observations. The network structures are described in Tab.~\ref{tab:fc network} and Tab.~\ref{tab:cnn}. An additional linear layer  is used to map the output into representations. We train the networks for 200,000 iterations by Adam optimizer~\cite{kingma2014adam} with batch size 1024 and learning rate 0.001. The weight for the penalty term in Eqn.~\eqref{eqn: penalty} is set to 1.0. Following~\cite{wu2018laplacian}, we use the discounted multi-step transitions with discount parameter 0.9.

For learning Laplacian representations on \texttt{PointRoom} and \texttt{PointMaze} environments, we collect a dataset of 1,000,000 transitions using a uniformly random policy with random starts. Each episode has a length of 500. We use the same fully connected network as mentioned above and keep other configurations unchanged except using a larger batch size of 8192.

For computing $\mathrm{SimGT}$ and $\mathrm{SimRUN}$ for continuous states, we calculate the inner summation in Eqn.~\eqref{eqn: simGT} and Eqn.~\eqref{eqn: simRUN} over sampled states rather than the entire state space.

\begin{table}[t]
    \centering
    \caption{Network architecture of the fully connected network.}
    \vskip 0.1in
    \begin{tabular}{c c c}
    \toprule
        Layer & Number of units & Activation  \\
    \midrule
        Linear & 256 & ReLU \\
        Linear & 256 & ReLU \\
        Linear & 256 & ReLU \\
    \bottomrule
    \end{tabular}
    \label{tab:fc network}
\end{table}

\begin{table}[t]
    \centering
    \caption{Network architecture of the convolutional network. (C, K, S, P) correspond to number of output channels, kernel size, stride and padding.}
    \vskip 0.1in
    \begin{tabular}{c c c}
    \toprule
        Layer & Configurations (C, K, S, P) & Activation  \\
    \midrule
        Conv2D & (16, 4, 2, 2) & ReLU \\
        Conv2D & (16, 4, 2, 2) & ReLU \\
        Conv2D & (16, 4, 1, 0) & ReLU \\
    \bottomrule
    \end{tabular}
    \label{tab:cnn}
\end{table}

\subsection*{E.2 Option Discovery}

\begin{table}[t!]
    \centering
    \caption{Hyperparameters of DQN for learning options.}
    \vskip 0.1in
    \begin{tabular}{c c}
    \toprule
        Timesteps & 100,000 \\
        Episode length & 50 \\
        Optimizer & Adam \\
        Learning rate & 1e-3 \\
        Learning starts & 5000 \\
        Training frequency & 1 \\
        Target update frequency & 50 \\
        Target update rate & 0.05 \\
        Replay size & 100,000 \\
        Batch size & 128 \\
        Discount factor $\gamma$ & 0 \\
    \bottomrule
    \end{tabular}
    \label{tab:params of dqn for options}
\end{table}

We run option discovery experiments on \texttt{GridRoom} and \texttt{GridMaze} environments with $(x,y)$ position observations. 
Following~\cite{machado2017laplacian}, we approximate the options greedily ($\gamma=0$). For each dimension of the learned representation, one option is trained by Deep Q-learning~\cite{mnih2013playing} with an intrinsic reward function $r_i(s,s')=f_i(s) - f_i(s')$ and the other with $-r_i(s,s')$. 
The termination set of an option is defined as the set of states where $f_i(s)$ is a local maximum (or minimum for the other direction). For the deep Q-network (DQN), we use the same fully connected network as one used for learning representations. The hyperparameters for training DQN are summarized in Tab.~\ref{tab:params of dqn for options}.

To compute $N_{i\to j}$, we first augment the agent's action space with the learned options. For each starting state in room $i$, we record how many steps an agent takes to arrive in room $j$ when it follows a uniformly random policy. We run 50 trajectories for each starting state to stabilize the result.

\subsection*{E.3 Reward Shaping}

\begin{figure}[t]
    \centering
    \includegraphics[width=0.6\linewidth]{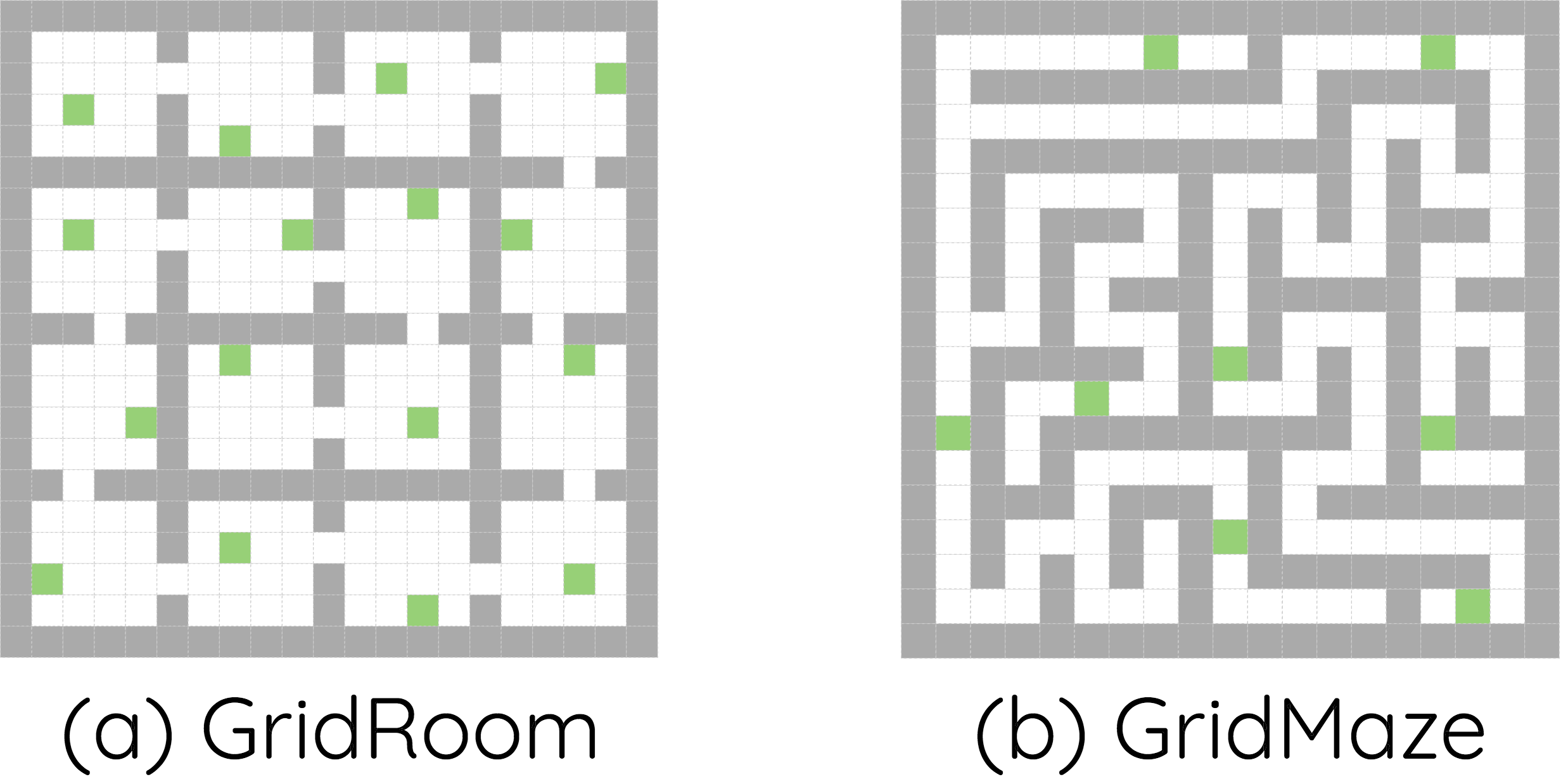}
    \vskip -0.05in
    \caption{Goal positions in \texttt{GridRoom} and \texttt{GridMaze} for reward shaping experiments. Each green cell represents a goal.}
    \vskip -0.1in
     \label{fig:goal positions}
\end{figure}

\begin{table}[t]
    \centering
    \caption{Hyperparameters of DQN for reward shaping.}
    \vskip 0.1in
    \begin{tabular}{c c}
    \toprule
        Timesteps & 200,000 \\
        Episode length & 150 \\
        Optimizer & Adam \\
        Learning rate & 1e-3 \\
        Learning starts & 5000 \\
        Training frequency & 1 \\
        Target update frequency & 50 \\
        Target update rate & 0.05 \\
        Replay size & 100,000 \\
        Batch size & 128 \\
        Discount factor $\gamma$ & 0.99 \\
    \bottomrule
    \end{tabular}
    \label{tab:params of dqn for reward shaping}
\end{table}

We run reward shaping experiments on \texttt{GridRoom} and \texttt{GridMaze} environments. Following~\cite{wu2018laplacian}, we train the agent in goal-achieving tasks using Deep Q-learning~\cite{mnih2013playing} with $(x,y)$ positions as observations. At each step, the agent receives a reward of 0 if it reaches the goal state and -1 otherwise. The success rate of reaching the goal state is used to measure the performance. As mentioned in the main paper, we use multiple goals to eliminate the bias brought by the goal position. Their locations are depicted in Fig.~\ref{fig:goal positions}. For the Q-network, we use the same fully connected network as one used for learning representations. The hyperparameters for training DQN are summarized in Tab.~\ref{tab:params of dqn for reward shaping}.


\section*{F. Additional Results}
\subsection*{F.1 Learning Laplacian Representations}

In Sec.~\ref{subsec: exp:laprep}, Fig.~\ref{fig:vis-gridmaze} and Fig.~\ref{fig:vis-pointroom} visualize the learned representations on \texttt{GridMaze} and \texttt{PointRoom}. Here we include additional visualizations for \texttt{GridRoom} and \texttt{PointMaze} in Fig.~\ref{fig:vis-gridroom} and Fig.~\ref{fig:vis-pointmaze}.

In Sec.~\ref{subsec: exp:laprep}, Fig.~\ref{fig:large variance} visualize first 3 dimensions of learned representations in different runs on \texttt{GridRoom}. Here we show all 10 dimensions in Fig.~\ref{fig:large variance full ours} and Fig.~\ref{fig:large variance full gd}.

\begin{figure}[ht!]
     \centering
    \includegraphics[width=\linewidth]{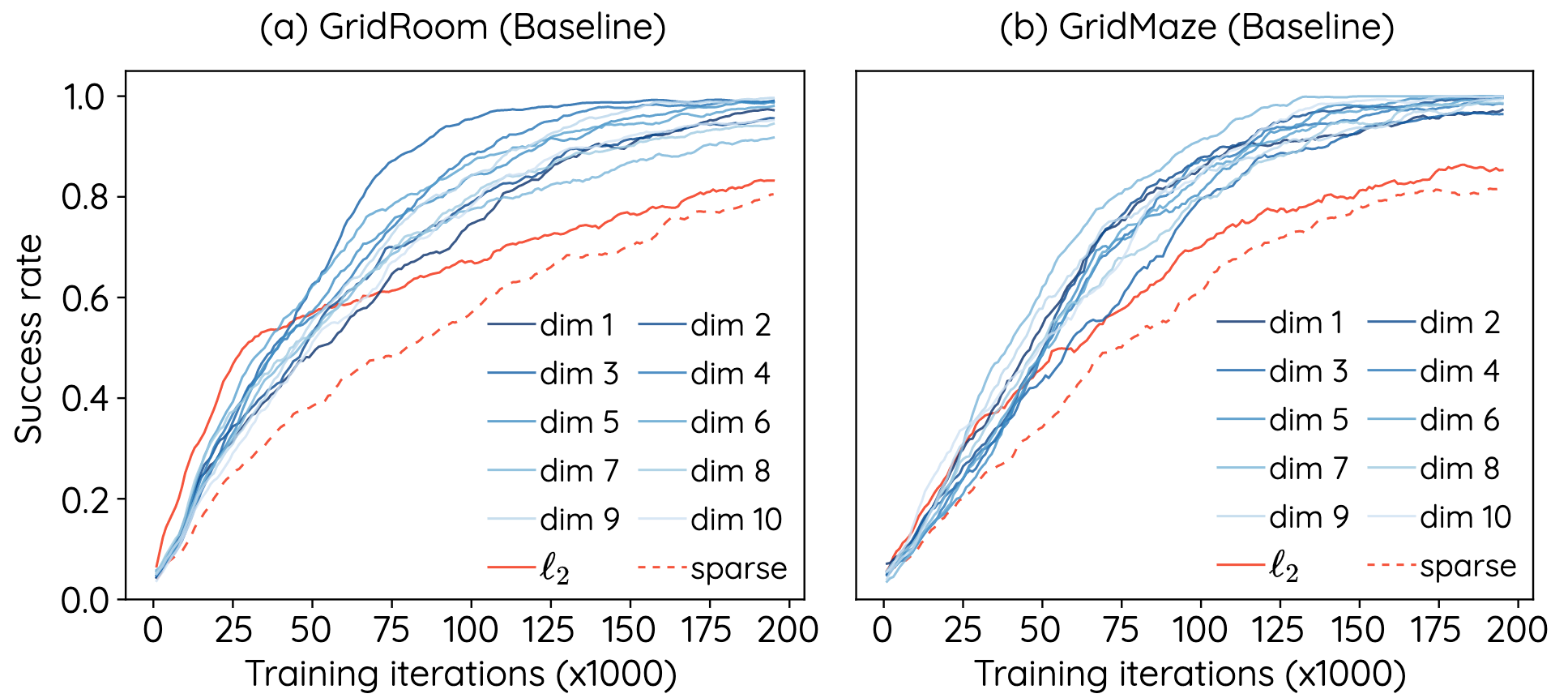}
    \vskip -0.05in
    \caption{Results of reward shaping with each dimension of Laplacian representations learned by baseline method. $\ell_2$ denotes reward shaping with L2 distance in raw observation space (\ie, $(x,y)$ position), and \emph{sparse} denotes no reward shaping.}
    \vskip -0.1in
     \label{fig:reward shaping baseline}
\end{figure}

\begin{figure}[ht!]
     \centering
    \includegraphics[width=\linewidth]{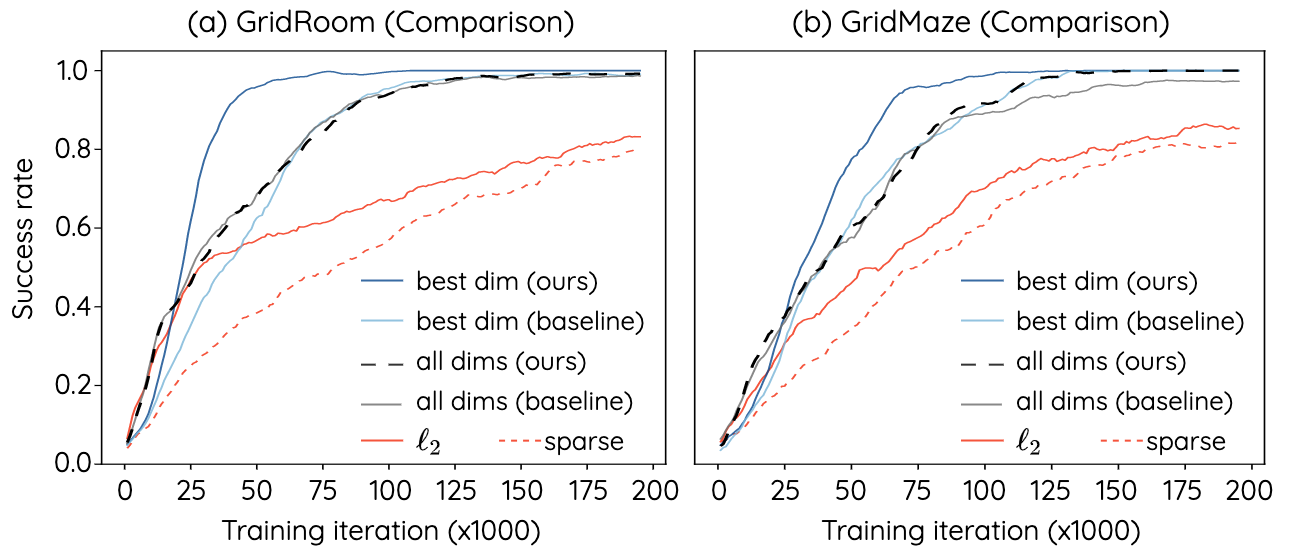}
    \vskip -0.05in
    \caption{Results of reward shaping with learned Laplacian representations. $\ell_2$ denotes reward shaping with L2 distance in raw observation space (\ie, $(x,y)$ position), and \emph{sparse} denotes no reward shaping.}
    \vskip -0.1in
     \label{fig:reward shaping supp}
\end{figure}

\begin{table}[ht!]
    \centering
    \caption{Absolute cosine similarity (averaged across dimensions) between our learned representation and ground truth, on \texttt{GridRoom} environment.}
    \vskip 0.1in
    \begin{tabular}{c c}
    \toprule
        Coefficients & Similarity \\
    \midrule
        group 1 & 0.9905 \\
        group 2 & 0.9653 \\
        default & 0.9913 \\
    \bottomrule
    \end{tabular}
    \label{tab:simsum_gridroom}
\end{table}

\subsection*{F.2 Reward Shaping}

For completeness, we show the results with each dimension of learned representation for baseline method in Fig.~\ref{fig:reward shaping baseline}, and include the results for “all dims - ours” in Fig.~\ref{fig:reward shaping supp}.

\subsection*{F.3 Evaluation On Other Coefficient Choices}

In Sec.~\ref{subsubsec:effectiveness of generalized objectives}, Fig.~\ref{fig:coefficients} shows the similarities between our learned representation (with different coefficient groups) and the ground truth on \texttt{GridMaze}. Here we show the results on \texttt{GridRoom} in Tab.~\ref{tab:simsum_gridroom}.

\subsection*{F.4 Visualization of the discovered options}

In Fig.~\ref{fig:viz option gridmaze} and \ref{fig:viz option gridmaze contd}, we visualize the discovered options by different representations.

\begin{figure*}[t]
     \centering
    \includegraphics[width=\linewidth]{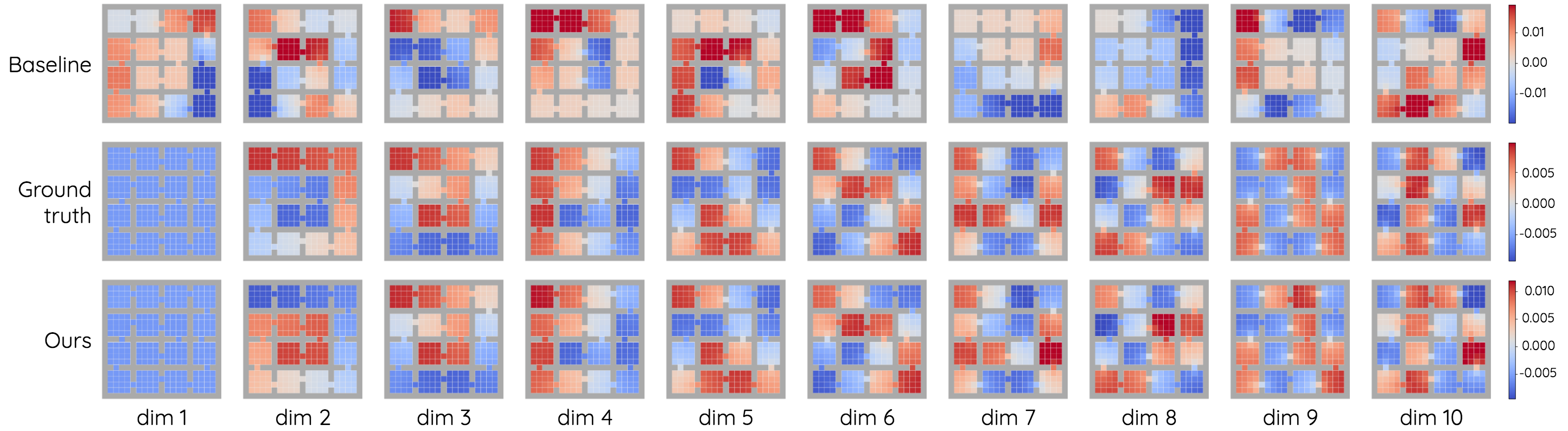}
    \vskip -0.05in
    \caption{Visualization of the learned 10-dimension Laplacian representation and the ground truth on \texttt{GridRoom}. Each heatmap shows a dimension of the representation for all states in the environment. Best viewed in color.}
     \label{fig:vis-gridroom}
\end{figure*}

\begin{figure*}[t]
     \centering
    \includegraphics[width=\linewidth]{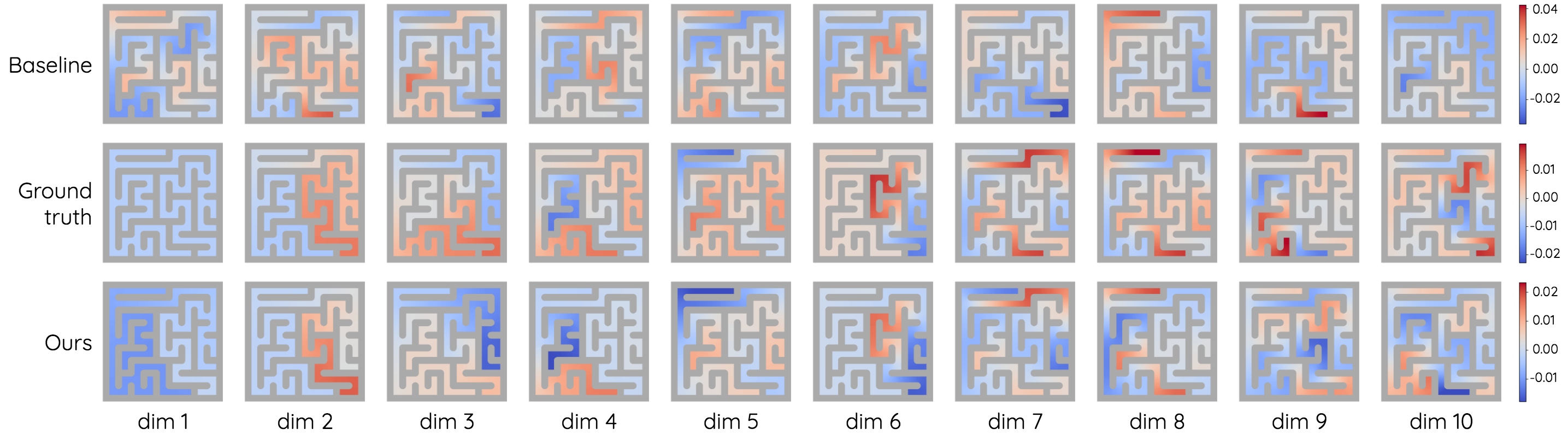}
    \vskip -0.05in
    \caption{Visualization of the learned 10-dimension Laplacian representations and the ground truth on \texttt{PointMaze}. Each heatmap shows a dimension of the representation for all the states in the environment. Best viewed in color.}
     \label{fig:vis-pointmaze}
     \vskip -0.05in
\end{figure*}

\clearpage
\begin{figure*}[t]
     \centering
    \includegraphics[width=0.9\linewidth]{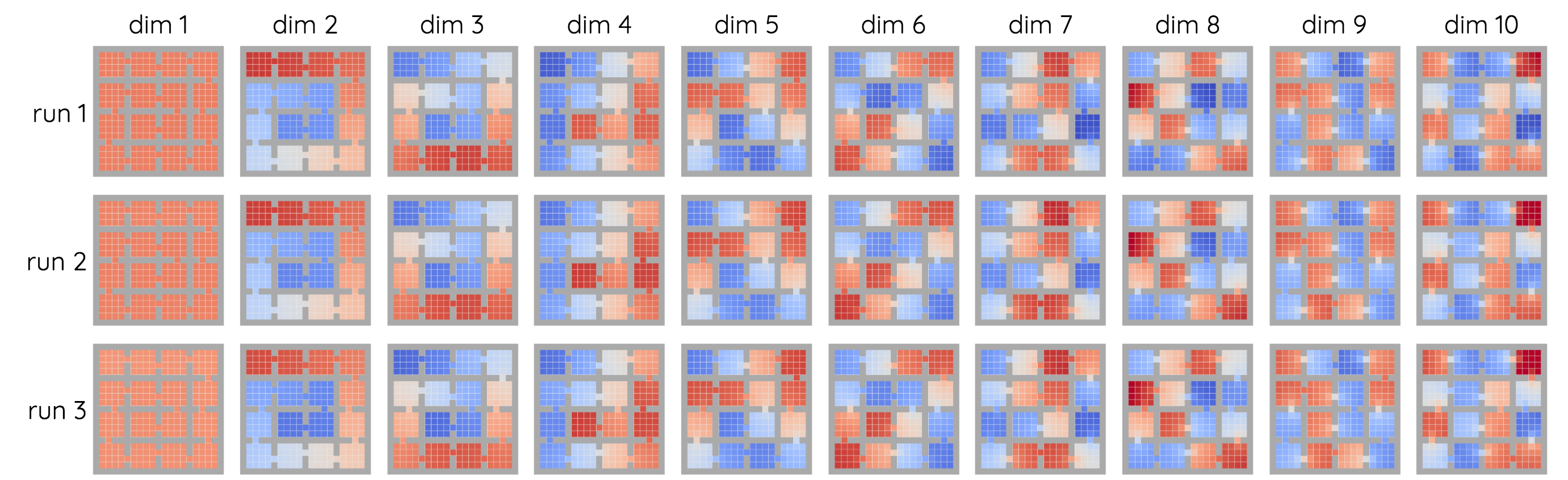}
    \vskip -0.05in
    \caption{Visualization of the Laplacian representations learned by our method on \texttt{GridRoom} in 3 different runs.}
     \label{fig:large variance full ours}
     \vskip -0.05in
\end{figure*}

\begin{figure*}[t]
     \centering
    \includegraphics[width=0.9\linewidth]{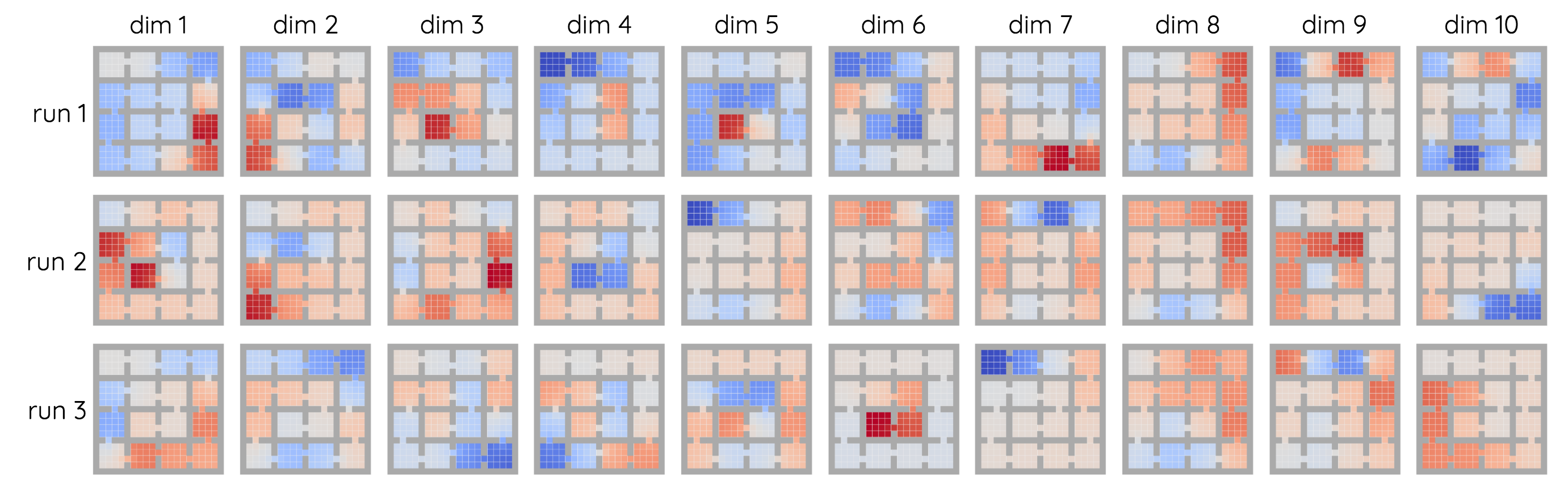}
    \vskip -0.05in
    \caption{Visualization of the Laplacian representations learned by baseline method \texttt{GridRoom} in 3 different runs.}
     \label{fig:large variance full gd}
     \vskip -0.05in
\end{figure*}

\begin{figure*}[t]
     \centering
    \includegraphics[width=0.77\linewidth]{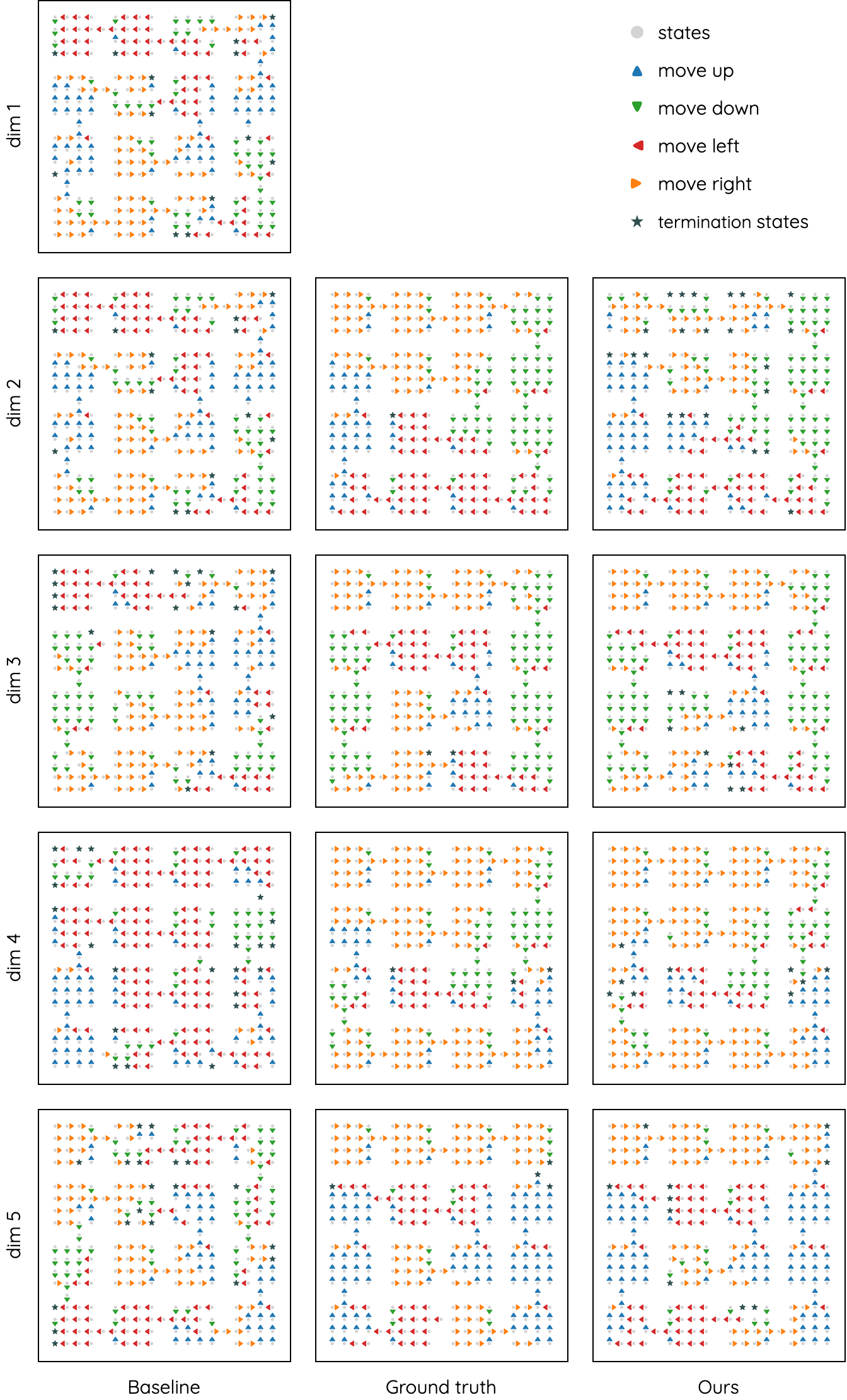}
    \vskip -0.05in
    \caption{Visualization of the discovered options in \texttt{GridRoom}.}
     \label{fig:viz option gridmaze}
     \vskip -0.05in
\end{figure*}

\begin{figure*}[t]
     \centering
    \includegraphics[width=0.77\linewidth]{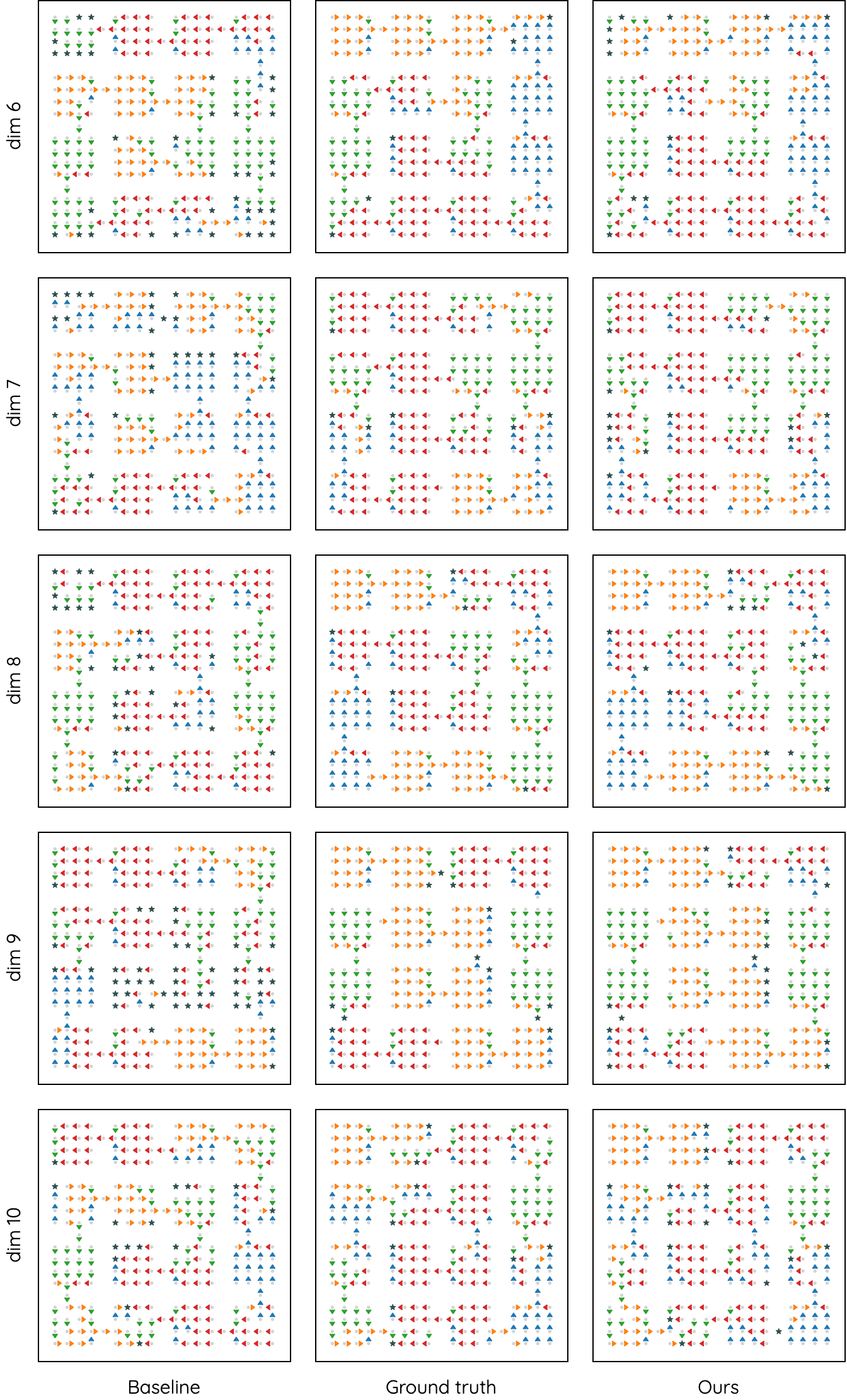}
    \vskip -0.05in
    \caption{Visualization of the discovered options in \texttt{GridRoom} (continued).}
     \label{fig:viz option gridmaze contd}
     \vskip -0.05in
\end{figure*}

\end{document}